\documentclass[a4paper, 11pt]{article}

\usepackage[left=2.7cm, right=2.7cm, top=3cm, bottom=3cm]{geometry}
\usepackage{amsmath, amssymb, epsfig, amsfonts, mathtools, latexsym, mathrsfs, amsthm}
\usepackage{pgfplots, subcaption}
\usepackage{multirow}
 \usepackage{graphicx}
\usepackage{booktabs}
\usepackage{float}
\usepackage[normalem]{ulem}
\pgfplotsset{compat=1.16}
\usepgfplotslibrary{fillbetween}
\usepackage{hyperref}
\hypersetup{
    colorlinks,
    citecolor=blue,
    filecolor=blue,
    linkcolor=blue,
    urlcolor=blue
}
\usepackage{url}

\usepackage[style=ieee, citestyle=numeric-comp, backend=biber]{biblatex} 
\usepackage{blindtext}
\usepackage{braket}
\usepackage{mathtools}
\usepackage{xparse}
\usepackage[font=footnotesize]{caption}

\usepackage{tikz}
\usepackage{tikz-cd}
\usepackage{diagbox}
\usetikzlibrary{positioning}
\usetikzlibrary{calc}
\usetikzlibrary{decorations.pathreplacing,calligraphy}
\usetikzlibrary{arrows, arrows.meta} 
\usetikzlibrary{shapes, matrix} 

\newcommand{\be}{\begin{equation}} 
\newcommand{\ee}{\end{equation}}
\newcommand{\bea}{\begin{equation} \begin{aligned}} \newcommand{\eea}{\end{aligned} \end{equation}}

\newcommand{\cN}{\mathcal{N}}
\newcommand{\cO}{\mathcal{O}}
\newcommand{\cS}{\mathcal{S}}
\newcommand{\bN}{\mathbb{N}}

\newcommand{\ov}{\over}

\newcommand{\ie}{\textit{i.e.}\ }
\newcommand{\eg}{\textit{e.g.}\ }

\DeclarePairedDelimiter{\ceil}{\lceil}{\rceil}
\DeclarePairedDelimiter{\floor}{\lfloor}{\rfloor}

\newcommand{\rr}{\mathcal{R}}
\newcommand{\bb}{\mathcal{B}}

\newtheorem{theorem}{Theorem}
\newtheorem{definition}{Definition}[section]

\newtheorem{remark}[definition]{Remark}
\newtheorem{corollary}[definition]{Corollary}
\newtheorem{lemma}[definition]{Lemma}
\newtheorem{proposition}[definition]{Proposition}
\newtheorem{conjecture}[definition]{Conjecture}

\newtheorem*{inner}{\innerheader}
\newcommand{\innerheader}{}
\newenvironment{defi}[1]
 {\renewcommand\innerheader{#1}\begin{inner}}
 {\end{inner}}

\newenvironment{LateProof}{\noindent \textit{Proof}}{\hfill$\square$ \\}
\newenvironment{SketchProof}{\noindent \textit{Sketch of Proof.}}{\hfill$\square$ \\}

\newcommand{\@toptitlebar}{
  \hrule height 4\p@
  \vskip 0.25in
  \vskip -\parskip%
}
\newcommand{\@bottomtitlebar}{
  \vskip 0.29in
  \vskip -\parskip
  \hrule height 1\p@
  \vskip 0.09in%
}

\providecommand{\@maketitle}{}
\renewcommand{\@maketitle}{%
  \vbox{%
    \hsize\textwidth
    \linewidth\hsize
    \vskip 0.1in
    \@toptitlebar
    \centering
    {\LARGE\bf \@title\par}
    \@bottomtitlebar
      \begin{tabular}[t]{c}\bf\rule{\z@}{24\p@}\@author\end{tabular}%
    \vskip 0.3in \@minus 0.1in
  }
}

\makeatletter
\def\and{
  \end{tabular}%
  \hskip 4em \@plus.17fil%
  \begin{tabular}[t]{c}}
\makeatother

\newcommand{\ignore}[1]{}
\usepackage{bbm}

\begin{document}

\title{Consensus learning: A novel decentralised ensemble learning paradigm}
\author{\textbf{Horia Magureanu}\\ 
Flare Research\\
\texttt{horia@flare.foundation} 
\and
\textbf{Na\"{i}ri Usher}\\
Flare Research\\
\texttt{nairi@flare.foundation}}

\date{ }
\maketitle 

\begin{abstract}
The widespread adoption of large-scale machine learning models in recent years highlights the need for distributed computing for efficiency and scalability. This work introduces a novel distributed machine learning paradigm -- \emph{consensus learning} -- which combines classical ensemble methods with consensus protocols deployed in peer-to-peer systems. These algorithms consist of two phases: first, participants develop their models and submit predictions for any new data inputs; second, the individual predictions are used as inputs for a communication phase, which is governed by a consensus protocol. Consensus learning ensures user data privacy, while also inheriting the safety measures against Byzantine attacks from the underlying consensus mechanism. We provide a detailed theoretical analysis for a particular consensus protocol and compare the performance of the consensus learning ensemble with centralised ensemble learning algorithms. The discussion is supplemented by various numerical simulations, which describe the robustness of the algorithms against Byzantine participants.
\end{abstract}

\setcounter{tocdepth}{2}
\tableofcontents

\numberwithin{equation}{section}  

\section{Introduction}

Machine learning (ML) has traditionally existed within the context of centralised computing, whereby both data processing and computations occur on a single server. More recently, distributed settings~\cite{konevcny2015federated, mcmahan2017communication, abadi2016tensorflow, ben2019demystifying, 10.1145/3377454} have garnered increased attention due the complexity of modern foundation models, such as large language and computer vision models~\cite{bommasani2022opportunities}, which require vast quantities of data to be processed. There, both data and computational resources can be spread across multiple devices or nodes. A prominent distributed learning paradigm is federated learning (FL), where nodes train a model collectively by sharing only local model updates in order to protect data privacy~\cite{konevcny2015federated, mcmahan2017communication, konevcny2016federated}.

Yet, FL algorithms and more generally distributed algorithms are vulnerable to malicious or faulty behaviour – termed Byzantine behaviour – of the participants~\cite{10.1145/357172.357176}, and dealing with such participants is one of the most challenging problems in distributed ML~\cite{10.1145/3616537, shejwalkar2021manipulating, bouhata2022byzantine, shi2022challenges}. The resilience against such Byzantine actors is rooted in the aggregation rule used to combine the local model updates shared in every federated training round~\cite{10.1145/3616537}. Furthermore, even though FL algorithms can leverage encryption or differential privacy mechanisms, they remain susceptible to privacy threats \cite{cheu2019manipulation}, as the local model updates have been shown to contain enough information to reconstruct local data samples~\cite{NEURIPS2020_c4ede56b, NEURIPS2019_60a6c400, fredrikson2015model}.

Another issue encountered in many distributed learning settings, and predominantly in popular FL methods, is the reliance on a central server, which effectively restricts such algorithms to an enterprise-only setting. Fully decentralised, or peer-to-peer algorithms, operate without a central authority and depend instead on the network topology~\cite{warnat2021swarm, 8186925}. Nevertheless, in such cases enhanced robustness against Byzantine users is typically achieved only for dense network topologies, which leads to increased communication overhead~\cite{NEURIPS2021_d2cd33e9}. This challenge is already evident for training or fine-tuning of large models within the FL framework, even in the absence of malicious players~\cite{woisetschläger2024survey}.

Ensemble methods combine the knowledge of multiple models, built with different architectures, parameters, and amount of available data, to solve a single task~\cite{sagi2018ensemble}. More precisely, ensembles are built on the strengths of every contributor in order to overcome the weaknesses of the individual models. These methods enlarge the representation space of a model, which is also one of the goals of transfer learning methods~\cite{5288526, singh1992transfer, thrun1998learning, Baxter_2000}, notably deployed in foundation models~\cite{bommasani2022opportunities}. Ensemble methods have been typically considered in a centralised setting, and thus implicitly assume that none of the participants are Byzantine.  The problem of Byzantine participants has recently been popularised through the widespread adoption of blockchains~\cite{nakamoto2009bitcoin}, where network participants vote on the validity of a set of transactions to reach agreement. This is implemented through a consensus protocol, which aims to provide both safety and liveness guarantees. Informally, this means that the system is able to deal with misbehaving participants. 

In this work, we introduce a novel ML paradigm that combines consensus protocols with common ensemble learning methods, which we naturally term \emph{consensus learning}. As opposed to the FL setup, individual participants do not share information about their ML models, nor any local data, which allows us to bypass data privacy and leakage issues. Instead, participants are required to only share their predictions for any given data inputs of a test dataset, to which any participant may contribute. To strengthen Byzantine robustness, these predictions then enter a communication phase, which is governed by a consensus protocol to ensure that the network reaches agreement. Here, honest participants will truthfully follow the rules of the consensus protocol, while malicious ones will attempt to stir the network to their desired outputs.

\subsection{Main contributions}

Consensus learning enhances typical ensemble weighting methods through a communication phase, where participants share their outputs until consensus is reached. We present a theoretical analysis of the performance of a consensus learning algorithm specialised to a binary classification task, which represents a toy model suitable for explaining the subtleties of this novel paradigm. More precisely, we deploy the Slush consensus protocol from the Snow family of protocols~\cite{rocket2020scalable}, which is a family of gossip protocols.

Our analysis provides lower bounds on the accuracy of a binary classifier deploying consensus learning and indicates which practical settings are suitable for using this type of algorithms. Moreover, we compare the Slush consensus learning algorithm with simple ensemble methods such as centralised majority rules and describe scenarios where the former can be the better performer. This analysis is supplemented by various numerical simulations, which describe the resilience of consensus learning against Byzantine participants.

While this work focuses mostly on classification tasks, consensus learning algorithms can also be applied to regression problems. There, robust local aggregation rules need to be deployed, similarly to Byzantine ML algorithms~\cite{10.1145/3616537}. We also comment briefly on the applicability of consensus learning to unsupervised or self-supervised learning, and leave a detailed analysis of these use cases for future work.

\subsection{Related works} \label{sec: related works}

Consensus learning is closely related to meta-learning~\cite{9428530}, also known as learning-to-learn~\cite{thrun1998learning}. These methods involve a meta-learner, which is an ML model trained on the outputs of the base learners. In fact, weighting methods where the weights are based on the precision of the base learners can already be thought of as simple meta-learning methods. The first stage of consensus learning methods is, indeed, identical to that of meta-learning. The communication phase, however, does not involve a secondary training round per se; regardless, the analogy with meta-learning methods might stem from the local aggregation rules used in this phase, which can involve weighting methods.

Meta-learning methods have recently been used for unsupervised learning~\cite{metz2019metalearning}, indicating that consensus learning methods could be applied in such contexts as well. Of particular relevance would be unsupervised ensemble methods, such as consensus clustering~\cite{strehl2002cluster}, which offer a natural playground for extending consensus learning methods. Note that consensus clustering algorithms combine clusterings from multiple sources without accessing the private data of individual participants. Nevertheless, this aggregation is done in a centralised fashion through a \emph{consensus function}, which also assumes that participants are honest. A peer-to-peer adaptation can be implemented by adding a communication phase, where the consensus function would be used for local aggregations.

Another related approach to consensus learning is federated one-shot learning, which is an FL method that allows the central server to learn a model in a single communication round~\cite{guha2019oneshot, NEURIPS2022_868f2266}. Federated one-shot learning could be also classified as an ensemble method with some additional features, such as cross-validation user selection: only users achieving a baseline performance on some validation data can be part of the global model. Consensus protocols are generally not feasible for multi-round federated learning algorithms, due to the amount of computing resources required for obtaining satisfying results~\cite{NEURIPS2021_d2cd33e9}, but could be used for one or few-shot FL. Weaker forms of consensus, such as approximate agreement~\cite{10.1145/5925.5931} and averaging agreement~\cite{NEURIPS2021_d2cd33e9} have been argued to be more suitable for modern ML applications.

Knowledge distillation (KD) \cite{hinton2015distilling} based methods also share same similarities with consensus learning. There, the goal is to compress the knowledge of a group of \emph{teacher models} into a smaller \emph{student model} that can approximate the teachers' predictions with high precision. In KD-based FL algorithms, participants only communicate their predictions on an unlabelled public test set, which are then used for improving local models \cite{chang2019cronus} -- see also \eg \cite{li2019fedmd, roux2024byzantineresilience} and references therein for other recent works on the topic. Note that most research in this direction focuses on centralised algorithms.

Consensus algorithms have recently been deployed in distributed neural networks. In this setup, data is split among a number of agents, which share the same initialised neural networks. These models are then trained locally, with the local updates aggregated based on the network topology and a deterministic consensus algorithm. The setup is somewhat similar to peer-to-peer FL, and has been shown to achieve convergence using a small number of communication rounds after each training phase -- see \eg \cite{liu2019distributed_v0}. This setup becomes closer to our approach when using a heuristic adaptive consensus algorithm, which deploys local aggregation functions with varying weights~\cite{liu2019distributed}. This approach is based on switching communication graphs for the network topology, which is reminiscent of temporal graph neural networks~\cite{10.5555/3455716.3455786}. 

Blockchain mechanisms, such as \emph{proof-of-learning} (PoL)~\cite{8783030}, also appear to share some of the features of consensus learning. The implementation described in~\cite{8783030}, which was inspired by Kaggle machine learning competitions, involves a set of nodes called \emph{trainers} that submit ML models for tasks previously set by other nodes, referred to as \emph{suppliers}. These models are ranked according to their performance on unseen data by a set of randomly selected \emph{validators}. A similar proposal was sketched out in~\cite{LIU2021108594}. Nevertheless, these PoL mechanisms still lack a basic utility: rather than aiming to collaboratively solve a problem, each node only tries to have the best model on their own. Blockchain empowered FL methods, such as those of~\cite{warnat2021swarm, 8998397, 9127823}, align more closely with our proposal; regardless, such blockchain implementations of FL methods are currently not feasible due to the shortcomings of EVM-based chains.\footnote{EVM chains do not natively support floating point numbers.} Additionally, certain client selection algorithms might be needed to ensure a minimum algorithm precision.

\subsection{Organisation}

The rest of the paper is organised as follows. In Section~\ref{sec: background} we cover some background material on jury problems and ensemble learning, focusing on binary classification problems. We also review consensus mechanisms and their fundamental properties. Section~\ref{sec: consensus learning} introduces the principles of consensus learning and summarises the key results of our work. Section~\ref{sec: results} provides a theoretical analysis for binary classifiers built with consensus learning algorithms, as well as a performance analysis in the presence of Byzantine learners. This analysis is extended in Section~\ref{sec: simulation} through numerical simulations on non-iid\footnote{Independent and identically distributed.} data. Finally, we summarise our findings and future research directions in Section~\ref{sec: conclusions}.

\section{Preliminaries}\label{sec: background}

\subsection{Ensemble learning} \label{sec: ensemble learning review}

Ensemble methods provide powerful techniques for combining multiple ML models to make a common decision~\cite{10.1007/3-540-45014-9_1}. Algorithms developed using ensemble learning techniques are task-agnostic, thus generalising across a wide range of problems (see \eg\cite{finn2017model, NEURIPS2020_24389bfe}). 

\begin{definition}[Ensemble]
    An ensemble is a collection of ML models whose predictions are combined together into a single output for any given input.
\end{definition}

\begin{definition}[Base learner]
    A base learner is one of the individual components of an ensemble.
\end{definition}

A base learner is thus a trained ML model deployed by a single participant of a network. The main premise of ensemble methods is that the errors of a single learner are compensated by the other learners~\cite{sagi2018ensemble}. In a distributed setting, such methods lead to significantly higher computational power and larger training datasets. Moreover, these techniques can reduce the risk of overfitting and increase the robustness of a model~\cite{10.1214/aos/1024691352, dietterich2002ensemble, 1688199}.

The effectiveness of ensemble learning techniques can also be appreciated through the perspective of hypothesis spaces~\cite{Blockeel2010}. Model building in supervised learning algorithms generally involves a search through a task-dependent hypothesis space, which can be understood as the set of functions between the input features and the output labels. In many instances, it is highly likely that the optimal hypothesis lies outside the hypothesis space of a single model. By combining multiple models, ensembles enlarge these individual spaces, thus increasing the likelihood of finding the optimal hypothesis~\cite{dietterich2002ensemble}. 

This idea is closely related to the concept of \emph{domain generalisation}, where multiple data sources are combined in order to improve a model's generalisation performance on unseen target domains~\cite{ben2010theory}. Additionally, ensemble learning techniques can achieve similar feats to \emph{transfer learning} methods~\cite{singh1992transfer, thrun1998learning, Baxter_2000}, whose goal is to leverage knowledge from different but related problems. The connection to these two concepts appears to be more pronounced especially when the base learners of an ensemble display significant dissimilarities, which will be a recurring theme of this work.

In ensemble learning, the outputs of the base learners can be combined together with two main methods: \emph{weighting methods} and \emph{meta-learning methods}. Typically, weighting methods turn out to be most suitable for cases when the performances of the learners are comparable. The simplest such method is a majority voting, where the assigned weights are all equal, and is commonly deployed in bootstrap aggregating (bagging)~\cite{breiman1996bagging}, or random forest~\cite{breiman2001random} algorithms. More intricate weight assignments are used, for instance, in boosting algorithms~\cite{schapire1998improved}.

Meta-learning algorithms, on the other hand, use a two-stage process in which the outputs of the base learners become inputs for an additional learning process~\cite{wolpert1992stacked}. Such methods are expected to perform extremely well when the base models have different performances on distinct subspaces of the dataset. Common meta-learning algorithms, such as \emph{stacking}~\cite{wolpert1992stacked}, rely on a central server - the \emph{meta-learner} - which is trained on the outputs of the base learners. Notably, both weighting methods and meta-learning algorithms are prone to various types of attacks from external users when considered in a distributed setting, which we aim to address in this work.

\subsection{Jury problems}
Ensemble methods implicitly assume that the base learners are \emph{honest}. This assumption is partly relaxed in a decentralised setting, where a fraction of the base learners is allowed to be Byzantine.

\begin{definition}[Honest participant]
    An honest participant is a participant who follows the modelling process truthfully.
\end{definition}

Importantly, according to this definition, a low-performance model can still be labelled as honest.  One of the fundamental results in ensemble methods is \emph{Condorcet's jury theorem}, which, despite its assumptions, captures essential aspects for building ML ensembles. This result is tailored to binary classification tasks, which we will also focus on for the rest of this work unless otherwise stated.

Hansen and Salamon~\cite{58871} adapted the jury theorem to an ML context, by modelling a base learner as a Bernoulli trial $X_i$, with probability $p_i$ of correctly identifying the label of a given input, for $i = 1, \ldots, n$, where $n$ is the number of base learners. In simplest terms, these success probabilities measure the accuracy of the base learners for the binary classification task at hand.

\begin{definition}[Accuracy]\label{def: accuracy}
    The accuracy of a classifier is the fraction of correctly identified samples in a test set.
\end{definition}

This measure can then be extrapolated to new inputs. Namely, the accuracy of an ML classifier will approximately give the probability of correctly identifying a new input. Of course, this assumes that the new input comes from the same distribution as the test data. In classification tasks, accuracy is based on the use of a \emph{unit loss function}, which does not distinguish between false positives and false negatives. As such, accuracy may not be the best performance metric for imbalanced datasets, where metrics such as F1-score or area under the ROC curve provide better alternatives~\cite{frank2015regression}. For our generic setting, however, accuracy will provide a reasonable metric.

To introduce Condorcet's jury theorem, let us first define independence and homogeneity.

\begin{definition}[Independence]  \label{def: independence}
    Base learners are independent if their associated random variables $\{ X_i\}_{i = 1, \ldots, n}$, are pairwise independent. 
\end{definition}

\begin{definition}[Homogeneity]  \label{def: homogeneity}
   A group of base learners is called homogeneous if all participants have the same accuracy on a specific input.
\end{definition}

The jury theorem dates back to the 18th century, and constitutes a majority rule ensemble method~\cite{58871}: voters are given a binary choice and the collective decision corresponds to that of the majority. For simplicity, we assume that the number of voters is odd, to guarantee that a decision can always be made.

\begin{defi}{Condorcet's jury theorem}
    Given a homogeneous group of $n$ independent base learners, for $n$ odd, each having accuracy $p > {1\ov 2}$, the accuracy $\mathbb{P}_{\rm Maj}$ of the ensemble built using a majority rule satisfies
    \be \label{hom Condorcet probability}
    \mathbb{P}_{\rm Maj}(p, n) = \sum_{j = \ceil*{{n\ov 2}}}^n \binom{n}{j}p^j(1-p)^{n-j} ~\geq ~p~,
    \ee
    with equality for $n = 1$ or $p=1$ only. Moreover, in the limit $n \to \infty$, we have
    \be
    \mathbb{P}_{\rm Maj}(p, n) \to 1~.
    \ee
\end{defi}

We refer to \eg \cite{7a1b040b-f861-32e9-9956-8bbdd7097f6c} for a proof of the first statement. The convergence in the large $n$ limit follows from the law of large numbers and will be included in the proof of Proposition~\ref{prop: Slush lower bound}. The basic principle behind Condorcet's jury theorem is that of the \emph{wisdom of crowds}, \ie the knowledge of a crowd is larger than that of a single member. This is, of course, not true in general, but convergence theorems can still be proven for heterogeneous or correlated juries -- see \eg \cite{10.1214/aoms/1177728178, 7a1b040b-f861-32e9-9956-8bbdd7097f6c}. 

Despite their apparent simplicity, weighting methods are effectively used by many state-of-the-art ensemble methods, such as bagging~\cite{breiman1996bagging}  or boosting algorithms~\cite{freund1995desicion}. In fact, the simple majority rule used in Condorcet's jury theorem turns out to be a very powerful aggregation rule for problems with a high degree of homogeneity, \ie where the models of the base learners have similar performances. This was explicitly demonstrated by Nitzan and Paroush (Theorem~1 of~\cite{3ea5cd6d-9799-36bb-8e45-f724ea53f0e0}), who showed that the optimal (decisive) decision rule\footnote{In the context of binary jury theorems, a decisive decision rule is a rule that ensures a decision for any set of juror votes. \label{footnote: DDR}} is a weighted majority, with the weights solely determined by the base learner accuracy.

This result shows, in particular, that in the homogeneous setting, the majority rule will outperform any other aggregation rule. Additionally, the majority rule will still perform close to optimal if the variance of the distribution of accuracies is not too large. In this sense, we introduce the notion of \emph{diversity}, which we will revisit in subsequent sections.

\begin{definition}[Diversity]\label{def: diversity}
    The diversity of a group of base learners is defined as the spread of the distribution\footnote{This is also known as \emph{variability}.} of accuracies of the base learners.
\end{definition}

\subsection{Consensus mechanisms} \label{sec: consensus mechanisms}

Consensus protocols were introduced in the context of distributed computing~\cite{10.1145/571637.571640, 1335465, 10.5555/1972495}, to ensure a system's security, resilience, and dependability~\cite{cachin2017blockchain}, and today are at the heart of blockchains. 

Blockchains are public databases which are updated and shared across many nodes in a network. Transactional data is stored in groups known as blocks, through the use of a unique identifier called a block hash, which is the output of a cryptographic hash function. This hash value is then part of the data of the next block of transactions, which thus links the blocks together in a chain. As such, each block cryptographically references its parent, and, thus, each block contains information about \emph{all} previous blocks. As a result, the data in a block cannot change without changing all subsequent blocks, which would require the approval of the entire network. Every node in the network keeps a copy of the database and, consequently, every node must agree upon each new block and the current state of the chain. To accomplish this distributed agreement, blockchains use consensus mechanisms. 

A protocol can solve the aforementioned problem of consensus if a set of conditions are satisfied~\cite{10.5555/1972495, cachin2017blockchain, amoressesar2024analysis}: every honest node eventually decides on some value (\emph{termination}); if all nodes propose the same value, then all nodes decide on that value (\emph{validity}); no node decides twice (\emph{integrity}); no two honest nodes decide differently (\emph{agreement}). These conditions become highly non-trivial in the presence of Byzantine nodes, and typically further assumptions about the environment might be needed for liveness and safety guarantees. These will be further discussed in Section~\ref{sec: modelling assumptions}.

Our primary focus will be on probabilistic consensus protocols, which rely on random or probabilistic processes. One such example is Nakamoto consensus~\cite{nakamoto2009bitcoin}  currently used by the Bitcoin network. Other examples include \emph{gossip} protocols~\cite{demers1987epidemic, 10.1145/3548606.3560638}, which are a class of communication protocols used to circulate information within a network. Probabilistic protocols typically require less communication overhead and thus tend to provide scaling advantages over deterministic variants; conversely, they can be subject to weaker resilience against malicious participants~\cite{892324}.

\subsubsection{Snow consensus protocols}    \label{sec: Snow intro}

In the following, we consider the Snow family of consensus protocols~\cite{rocket2020scalable}. These operate by repeatedly sampling the network at random, and steering correct nodes towards a common outcome, being examples of gossip protocols. In this section, we describe some of the technical aspects of the Slush protocol, which is the simplest protocol from this family. We include some new analytic results, summarised in Lemma~\ref{lemma: Slush symmetry} and Remark~\ref{remark: Slush absorption monotonicity}. Other technical details about the protocol are explained in Appendix~\ref{app: Avalanche}. See also~\cite{amoressesar2022spring, amoressesar2024analysis} for further recent analysis.

Consider a network consisting of $n$ nodes and a binary query with the output choices formally labelled by two colours, red and blue. At the beginning of the Slush protocol, a node can either have one of the two coloured states or be in an uncoloured state. Then, each node samples the network at random, choosing $k$ nodes to which they send a query. Every node responds to a query with its colour. 

Once a node receives $k$ responses, where typically $k\ll n$, it updates its colour if a threshold of votes $\alpha > \left\lfloor{k\ov 2}\right\rfloor$ is reached. This process is repeated multiple times, and each node decides on the colour it ends up with after the last communication round. To ensure convergence, Slush needs $\cO\left(n \log k\right)$ rounds~\cite{amoressesar2024analysis}, which is considerably lower than the $\cO(n^2)$ rounds required in most deterministic protocols.

The dynamics of the Slush protocol can be modelled as a continuous-time Markov process~\cite{rocket2020scalable}. For now, let us assume that all $n$ participants are honest and that all nodes are in a coloured state. We will refer to $\cS$ as the state (or configuration) of the network at any given time. Without loss of generality, the state simply represents the number of blue nodes in the system and takes values in the set $\{0, \ldots, n\}$. The process has two absorbing states, all-red and all-blue, corresponding to the final decision taken by the network. 

\begin{definition}[Slush absorption rates]
    Let the absorption rates in the all-red (all-blue) state of the Slush protocol from the state with $b$ blue nodes be $\rr_b$ ($\bb_b$, respectively). These satisfy $\rr_b + \bb_b = 1$.\footnote{See also the remark at the end of Section~\ref{sec: modelling assumptions}.}
\end{definition}

The exact expressions for these absorption probabilities are given in Appendix~\ref{app: Avalanche},  Corollary~\ref{corr: exact zeta}. Based on these results, we make the following remark.

\begin{remark}  \label{remark: Slush absorption monotonicity}
    The absorption probability in the all-blue state, $\bb_b$, increases monotonically with the number of blue nodes, $b$.
\end{remark}

Additionally, it also follows that $\rr_b$ is a monotonically decreasing function with $b$. A new and important result that we will use throughout this paper follows. The proof can be found in Appendix~\ref{app: Avalanche}.

\begin{lemma}\label{lemma: Slush symmetry}
    The Slush protocol is symmetric as long as all participants are honest, \ie
    \be
    \rr_b = \bb_{n-b}~,
    \ee
    for any $b \in \{0, \ldots, n\}$. Moreover, the majority absorption probability satisfies:
    \be \label{Slush hom bound}
    \bb_b \geq {b\ov n}~,
    \ee
    for $b \geq \ceil*{{n\ov 2}}$. Equality occurs for all $b$ whenever $k = \alpha = 1$.
\end{lemma}

Figure~\ref{fig: Slush Absorption} shows an explicit plot of the absorption probability $\bb_b$ and some relevant bounds for it. While the bound \eqref{Slush hom bound} is not particularly strong, it will play an important role in our analysis in Section~\ref{sec: results}. Note that this Lemma no longer holds in the presence of Byzantine nodes. Finally, let us mention that $\rr_{b\, <\, \alpha} = 1$ (and $\rr_{b\, >\, n-\alpha} = 0$), as, in these cases, the threshold for accepting a query can only be reached for the red (respectively blue) colour.

 \begin{figure}[t] %
    \centering
    \includegraphics[width=0.65\textwidth]{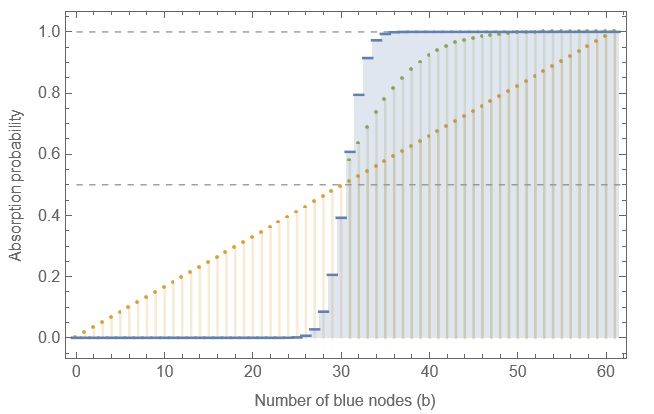}
    \caption{The $\bb_b$ absorption probability (in blue) for the Slush consensus protocol for $n=61$, $k=10$, $\alpha = 7$, as a function of the number of blue nodes $b$. In orange, the ${b\ov n}$ bound is plotted, while in green we have the Chvatal bound \eqref{Chvatal bound}.}\label{fig: Slush Absorption}
\end{figure}%

\section{Consensus learning}    \label{sec: consensus learning}

In this section, we introduce consensus learning, a fully distributed ML paradigm that is based on consensus protocols. We focus on supervised ML methods, and briefly comment on how the algorithm could be adapted to unsupervised and self-supervised problems.

\subsection{Algorithm description}

Supervised consensus learning is a two-stage process that can be described as follows. In this description, we assume for now the existence of a global test set.

\begin{enumerate}
    \item \textbf{Individual learning phase.} During the first stage, participants develop their own ML models, without the need to share any data or information about their models. At the end of this phase, participants determine their initial predictions for any given inputs. 
    \item \textbf{Communication phase.} During the communication phase, participants exchange their initial predictions on new inputs, and update them using a local aggregation function based on the outputs of the other base learners and their confidence in their own predictions. This phase is governed by a consensus protocol which may include several rounds, with the aim of guiding the network towards a common output. The outputs from the end of the communication phase will be the final outputs of the participants.
\end{enumerate}

\begin{figure}[!ht]
    \centering
    \begin{subfigure}{0.5\textwidth}
    \centering
    \includegraphics[width=1\textwidth]{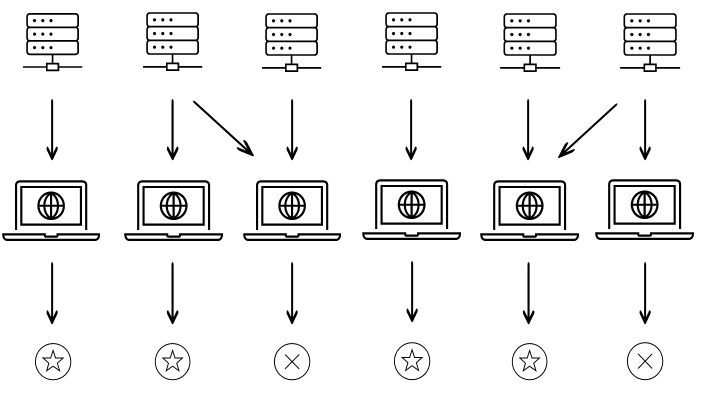}
    \caption{Individual training phase}
    \end{subfigure}%
    \begin{subfigure}{0.5\textwidth}
    \centering
    \includegraphics[width=0.8\textwidth]{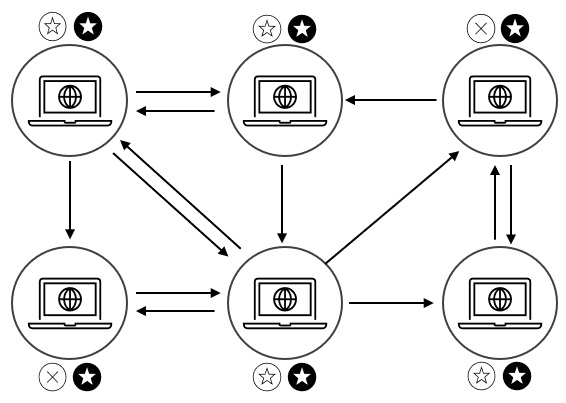}
    \caption{Communication phase}
    \end{subfigure}
\caption{Supervised consensus learning. (a) In the the first stage, participants develop their own models, based on datasets that may overlap. At the end of this phase, each model determines an initial prediction (hollow circles) for any new input. (b) In the communication phase, the initial outputs are exchanged between the participants, which eventually reach consensus on a single output (filled circles).}\label{fig: Consensus learning}
\end{figure}%

\noindent The two stages are depicted in Figure~\ref{fig: Consensus learning}. In an ideal case, the (honest) base learners reach consensus on a single global output in the communication phase. This unique output would then be the output of the ensemble formed by the base learners. However, this is not a requirement in fully decentralised algorithms, where participants may reach different final decisions. 

Consensus learning also allows for direct implementations on decentralised platforms such as blockchains. There, participants may either propose data for testing using proof-of-stake based protocols or use a predefined test set provided by an independent party. The latter can be facilitated through the use of a smart contract on Ethereum Virtual Machine (EVM) compatible blockchains.

The algorithm can also be adapted to self-supervised or unsupervised ML problems, where participants only have access to (partly) unlabelled data. For the former, each participant would deploy self-supervised learning techniques during the individual learning phase, such as contrastive learning~\cite{pmlr-v119-chen20j} or auto-associative learning~\cite{kramer1991nonlinear}. Then, the communication phase would proceed similarly to the supervised setting; here, the test set could include data inputs from the training sets of the individual participants, with the ensemble outputs being used to improve local models.

Unsupervised ensemble methods, such as consensus clustering~\cite{strehl2002cluster}, combine clusterings for multiple sources without accessing the private data of individual participants. This aggregation is commonly done in a centralised fashion through a \emph{consensus function}, which also assumes that participants are honest. Such methods can be adapted to peer-to-peer settings through the implementation of a communication phase, as described above. There, the consensus function would be used for local aggregations.

\subsection{Summary of key results}

Consensus learning methods are fully distributed ensemble techniques, being thus effective for generalising across a wide range of problems, as described in Section~\ref{sec: ensemble learning review}. Moreover, consensus learning is a much simpler type of meta-learning, being less demanding from a computational perspective. It is worth pointing out that meta-learning algorithms create additional sources of overfitting~\cite{NEURIPS2020_3e5190ee}, which are absent in consensus learning algorithms.

Another advantage of consensus learning is that it preserves the explainability of ensemble weighting methods. Thus, the relatively simple design of such algorithms offers transparency and interpretability. Importantly, consensus learning does not rely on a single central server. Additionally, adequate choices of probabilistic consensus protocols can result in low communication overhead, on par with that of centralised weighting methods. It is also natural to consider whether or not consensus-based methods can improve the performance of classical weighting methods, which we will discuss further below.

For classification tasks in supervised learning, perhaps the simplest weighting method is the equal-weight majority rule, deployed in methods such as Random Forests~\cite{breiman2001random}. In the next section, we consider binary classification and present a theoretical analysis of the performance of a consensus learning algorithm against centralised majority rules. For this, we deploy the Slush consensus protocol~\cite{rocket2020scalable} in the communication phase. We use accuracy as a performance metric (as per Definition~\ref{def: accuracy}) and analyse three types of scenarios, as discussed below.

\medskip 
\paragraph{I. Homogeneous scenario.} Arguably the simplest scenario to consider from an analytical point of view is one where all base learners have the same accuracy. In this homogeneous setting, Condorcet's jury theorem has long been one of the main pillars of ensemble learning. One of our most important results in this direction is a generalisation of Condorcet's jury theorem to a consensus learning algorithm.

\begin{theorem} \label{thm: Slush convergence}
    Consider a homogeneous group of $n$ independent base learners, with accuracies $p$ for a binary classification task. Then, the accuracy $\mathbb{P}_{\rm S}$ of the consensus learning algorithm using the Slush protocol satisfies:
    \be
        \mathbb{P}_{\rm S} \geq p ~,
    \ee
    for any $p > {1\ov 2}$,  with equality only occurring for $n = 1$ or $p=1$. Moreover, $\mathbb{P}_{\rm S}$  can be brought arbitrarily close to 1 for any $p > {1\ov 2} + {1\ov n}$, and large enough $n$.
\end{theorem}

The proof of this statement is rather involved, and is left to Appendix~\ref{app: proofs}. Nevertheless, the proof only uses simple features that are specific to the Slush protocol and could thus be generalised to other probabilistic consensus protocols. Other main results in the homogeneous setting are discussed in Section~\ref{sec: hom Slush} and include lower bounds on the accuracy of the consensus learning algorithm using the Slush protocol, as well as comparisons with majority and supermajority rules.\footnote{Supermajority rules are essential for proving Theorem~\ref{thm: Slush convergence}.}

\medskip 
\paragraph{II. Partly heterogeneous scenario.} Adaptations of jury theorems to heterogeneous juries have previously been discussed in the literature -- see \eg Theorem~4 of~\cite{10.1214/aoms/1177728178}, Theorem~3 of~\cite{10.2307/2348873}, as well as Theorem~3 of~\cite{fey2003note}. In fact, we expect that a similar result to Theorem~\ref{thm: Slush convergence} would hold for heterogeneous groups, but we do not explicitly pursue this direction. 

Instead, we compare consensus learning algorithms with majority rules in partly heterogeneous settings: more precisely, the accuracies of the base learners will be split into distinct homogeneous groups. We will see, in particular, that consensus learning algorithms can perform better than majority rules as long as the learners are \emph{diverse} enough, \ie when the distribution of accuracies of the base learners has a certain degree of heterogeneity, as per Definition~\ref{def: diversity}. This topic will be discussed in more detail in Section~\ref{sec: diversity}.

\medskip
\paragraph{III. Almost homogeneous scenario with Byzantine nodes.} An implicit assumption of well-established ensemble methods is that the base learners are honest. In a fully decentralised setting (or peer-to-peer), this assumption is relaxed, due to the presence of Byzantine participants. We clarify this notion in our framework below.

\begin{definition}[Byzantine participant] \label{def: Byzantine}    
    A malicious or Byzantine participant is a participant who is not honest. This also extends to the communication phase of a consensus learning algorithm.
\end{definition}

\noindent A Byzantine participant may share any outputs they wish during the communication phase. These will usually be decided based on their adversarial strategy. Section~\ref{sec: Byzantine} will present some analytical results regarding consensus learning algorithms with Byzantine users, where the honest base learners will form a homogeneous group. Our analysis comprises a comparison with majority and supermajority rules in the presence of Byzantine users.

\paragraph{Numerical results.} To better understand these scenarios and to provide more supporting evidence in favour of consensus learning algorithms, we will also present various numerical simulations in Section~\ref{sec: simulation}. We will use non-iid data from the LEAF benchmark~\cite{caldas2019leaf} and will analyse the effect of Byzantine users in more detail. Furthermore, we present slight modifications of the Slush protocol that can improve algorithm performance and Byzantine resilience.

\section{Theoretical analysis} \label{sec: results}

In this section we present a theoretical analysis of a consensus learning algorithm using the Slush consensus protocol, specialised to a binary classification problem. The underlying assumptions of this analysis are presented in Section~\ref{sec: modelling assumptions}. For ease of notation, we introduce the following terminology.

\begin{definition}[Slush algorithm]
    The Slush algorithm is the consensus learning method deploying the Slush consensus protocol in the communication phase.
\end{definition}

\subsection{Modelling premises} \label{sec: modelling assumptions}

To provide an in-depth analysis of a typical consensus learning algorithm, we will make some simplifying assumptions. Our analysis will be an extension of the Hansen and Salamon adaptation of jury theorems to an ML context~\cite{58871}. As such, we will model base learners as Bernoulli trials; the success probabilities of these random variables correspond to a chosen performance metric in the binary classification task. Another modelling assumption will be the independence of the base learners, as per Definition~\ref{def: independence}. We will relax this assumption in Section~\ref{sec: simulation}, where we conduct a simulation of a numerical consensus learning algorithm. Let us also mention that the number of base learners will typically be assumed odd unless otherwise stated.

For a concrete illustration of the communication phase, we will consider the Slush consensus protocol~\cite{rocket2020scalable}, briefly summarised in Section~\ref{sec: Snow intro}. We also refer to Appendix~\ref{app: Avalanche} for a more technical discussion of this protocol. This consensus protocol can be adapted to a binary classification task as follows. Consider a query, representing a data entry which needs to be classified by the ensemble. Initially, each node picks a class (which we refer to as \emph{colour}) as dictated by their local ML algorithm. For the binary classification problem of interest, we assume without loss of generality that class 1 (labelled as \emph{blue} colour) is the correct class of some given input to be classified, as opposed to class 0 (labelled as \emph{red} colour). Thus, at the beginning of the protocol, each node will be in a coloured state. These states can then change during the communication rounds.

Consensus protocols assume a form of \emph{synchronicity}, as agreement cannot be reached in a fully asynchronous setting~\cite{fischer1985impossibility}. For the technical analysis of the consensus learning algorithms, we will also assume a synchronous setting in the communication phase. Additionally, we will assume that consensus is eventually reached within the network, such that all (honest) base learners decide on the same final output.

Finally, let us briefly comment on the presence of Byzantine nodes. As pointed out in~\cite{rocket2020scalable}, if the number of Byzantine nodes is greater than the threshold $\alpha$ of accepting a query, then the Markov chain modelling the Slush protocol appears to have only a single absorbing state. Of course, in practical terms, the transition probabilities away from the all-blue state would be arbitrarily small. Nevertheless, throughout the paper, we will not consider this case. Another subtlety concerns the identity $\rr_b + \bb_b = 1$ for the absorption probabilities. It was recently argued that adversarial strategies that assume knowledge of the whole network can, in fact, indefinitely stall the protocol~\cite{MediumAvalancheLiveness}. Our analysis will be limited to fixed (extreme) adversarial strategies, where the Byzantine nodes will communicate the wrong class at all times. In such cases, the protocol can still be modelled as a Markov chain with fixed transition rates, thus ensuring that there is no closed communicating class, apart from the two absorbing states.

\subsection{Homogeneous case}\label{sec: hom Slush}

As a first scenario, we consider the homogeneous case, where each participant $j$ has the same accuracy in classifying a new input, $p_j = p$, for $j \in \{1, \ldots, n \}$, with $p \in [0,1]$. We will discuss the odd $n \in \bN$ case below.

\subsubsection{Majority rules}

In the homogeneous context, the first natural question is how a consensus learning algorithm will compare to a single implementation of a majority rule. Theorem~\ref{thm: Slush vs Majority hom} gives a first result in this direction.

\begin{theorem} \label{thm: Slush vs Majority hom}
    Given a homogeneous group of $n$ independent base learners, each with accuracy $p$ for a binary classification problem, the majority rule will outperform\footnote{That is, the accuracy of the majority rule is larger than that of the Slush algorithm.} the Slush algorithm, as long as $p>{1\ov2}$ and $\alpha \neq \ceil{{n\ov 2}}$. The Slush algorithm will achieve the same accuracy only for $\alpha = \ceil*{{n\ov 2}}$.
\end{theorem}

The veracity of this affirmation can be inferred from the Nitzan-Paroush theorem on optimal decision rules (\ie Theorem~1 of~\cite{3ea5cd6d-9799-36bb-8e45-f724ea53f0e0}). This states that the majority rule is the optimal decisive decision rule in the homogeneous setting. However, it is not entirely clear whether that theorem holds for our algorithm due to the existence of a communication phase. Thus, we provide below an alternative proof for this theorem.

\begin{proof}
    Note first that the number of blue nodes $b$ before the communication phase starts follows a binomial distribution. Thus, the probability of success for the Slush protocol is given by:
\bea    \label{Slush prob homogeneous}
     \mathbb{P}_{\text{S}}(n, k, \alpha, p) & =  \sum_{b=0}^{n}  \binom{n}{b} 
    \bb_b \, p^{b} (1-p)^{n-b} \\
   &  = \sum_{b=\alpha}^{n-\alpha}  \binom{n}{b} 
    \bb_b \, p^{b} (1-p)^{n-b} + \hspace*{-0.25cm} \sum_{b=n-\alpha+1}^n   \hspace*{-0.15cm} \binom{n}{b} p^{b} (1-p)^{n-b}~.
\eea
We would like to compare this expression with the expression \eqref{hom Condorcet probability} for the homogeneous majority rule. We immediately see that for $\alpha > {n \ov 2}$, the first sum in \eqref{Slush prob homogeneous} vanishes, and we have:
\be \label{large alpha}
    k \geq \alpha  \geq \ceil*{{n\ov 2}}: \qquad \mathbb{P}_{\text{S}}(n, k, \alpha , p) \leq \mathbb{P}_{\rm Maj}(p, n)~,
\ee
with equality only for $\alpha = \ceil*{{n\ov 2}}$. The more interesting case to analyse is $\alpha < {n\ov 2}$. For this, note that the first sum in the bottom line of \eqref{Slush prob homogeneous} can be further decomposed into
\be
   \sum_{b=\alpha}^{n-\alpha}  \binom{n}{b} 
    \bb_b\, p^{b} (1-p)^{n-b}  =  \sum_{b=\alpha}^{\floor*{{n\ov2}}} \binom{n}{b} 
    \bb_b\, p^{b} (1-p)^{n-b} + \sum^{n-\alpha}_{b=\ceil*{{n\ov2}}}\binom{n}{b} 
    \bb_b\, p^{b} (1-p)^{n-b}~.
\ee
Then, defining
\be \label{def DeltaP}
   \Delta \mathbb{P} = \mathbb{P}_{\rm Maj}(p, n) - \mathbb{P}_{\text{S}}(n, k, \alpha, p)~,
\ee
we find that for $n=2m+1$, $\Delta\mathbb{P}$ reduces to:
\be
   \Delta \mathbb{P} =  \sum_{b =m+1}^{n - \alpha}\binom{n}{b}\, \bigg( \rr_b \, p^b (1-p)^{n-b} - \bb_{n-b} \, p^{n-b} (1-p)^{b}\bigg)~.
\ee
To make further progress, we consider the ratio of the individual terms in the above summation:
\be
    \kappa_b \equiv {\rr_b \ov \bb_{n-b}} \times  \left( {p\ov 1- p} \right)^{2b-n}~,
\ee
for $m+1 \leq b \leq n-\alpha$. Using Lemma~\ref{lemma: Slush symmetry}, and since $2b > n$ for the range of interest, we have $\kappa_b>1$ as long as $p>{1\ov 2}$, with $\kappa_b = 1$ at $p = {1\ov2}$
and $\kappa_b<1$ otherwise. This concludes our proof.
\end{proof}

While the Slush algorithm cannot improve on the accuracy of an already optimal decisive decision rule in the homogeneous setting, we would still like to find a lower bound on its accuracy to illustrate its functionality. We thus seek a generalisation of Condorcet's jury theorem, which compares the accuracy of the ensemble generated through the Slush algorithm with that of a single base learner. The first part of this statement, illustrated by Theorem~\ref{thm: Slush convergence}, already gives such a lower bound, namely $\mathbb{P}_S \geq p$, for $p > {1\ov 2}$. In simple terms, this states that the ensemble built using this fully decentralised paradigm is more accurate than any of the base learners. A different bound on the Slush accuracy can be determined as follows.

\begin{proposition} \label{prop: Slush lower bound}
    Consider a homogeneous group of $n$ independent base learners, each with accuracy $p$. For large enough $n$, a lower bound for the accuracy of the Slush algorithm is given by
    \be
    \mathbb{P}_{\rm S}(n, k, \alpha, p) > 1 - e^{-2 \left({\alpha\ov k} - {1\ov 2} \right)^2k}~,
    \ee
    as long as $p>{1\ov 2}$.
\end{proposition}
\begin{proof}
    The proof of the statement uses the Chvatal tail bounds of the hypergeometric distribution~\cite{chvatal1979tail}. In particular, Theorem~1 of~\cite{rocket2020scalable} shows that
    \be \label{Chvatal bound}
        \rr_b \leq e^{-2\left({\alpha \ov k} - 1 + {b\ov n} \right)^2 k}~,
    \ee
    where $b \geq \ceil*{{n\ov 2}}$. Then, we have
    \be
    \mathbb{P}_{\rm S} > \sum_{b=\ceil*{{n\ov2}}}^n \mathbb{P}(S_n = b) \left( 1- \rr_b\right) \, > \,  \sum_{b=\ceil*{{n\ov2}}}^n \mathbb{P}(S_n = b) \left( 1- \rr_{\ceil*{{n\ov 2}}}\right)  ~,
    \ee
    with $S_n = \sum_{i=1}^n X_i$ being the sum of the Bernoulli trials associated to the base learners. For the last inequality, we use the fact that the value $b = \ceil*{{n\ov 2}}$ minimizes the expression $(1-\rr_b)$ for the range of the sum. Furthermore, it is not difficult to see that
    \bea
    \mathbb{P}\left(S_n \geq \ceil*{{n\ov 2}}\right) \, = \, \mathbb{P}\left({S_n \ov n} > {1\ov 2} \right)  \, \geq \,  \mathbb{P}\left(\left|{S_n \ov n}-p\right| < p-{1\ov 2} \right)~, 
    \eea
    which, by the weak law of large numbers, converges to 1 for large $n$. The result follows immediately.
\end{proof} 

Tail bounds are rather conservative, and thus the actual accuracy of the Slush algorithm is expected to be considerably better than this bound. Nonetheless, the bound provides a different perspective on the performance on the Slush algorithm. Specifically, for accurate classifiers with $p>{1\ov 2}$, this lower bound improves as the threshold parameter $\alpha$ for accepting a query increases. On the other hand, it was argued in~\cite{amoressesar2024analysis} that values of $\alpha$ closer to $k/2$ are more suitable for Byzantine consensus protocols. This is thus an important distinction between the objectives of consensus protocols in distributed computing and those of protocols designed for consensus ML. Nevertheless, Theorem~\ref{thm: Slush convergence} shows that consensus learning can leverage protocols primarily designed to safeguard distributed networks.

Other finite $n$ bounds, similar in spirit to that of Proposition~\ref{prop: Slush lower bound}, can be found using bounds for the binomial distribution~\cite{hoeffding1994probability}. However, we would like to find a stronger result for the case of large $n$. Ultimately, this search concludes with the second statement of Theorem~\ref{thm: Slush convergence}, which conveys that the Slush algorithm is indeed an efficient algorithm. The proof of the statement relies on supermajority rules, which we introduce next.

\begin{figure}[t] %
    \centering
    \includegraphics[width=0.65\textwidth]{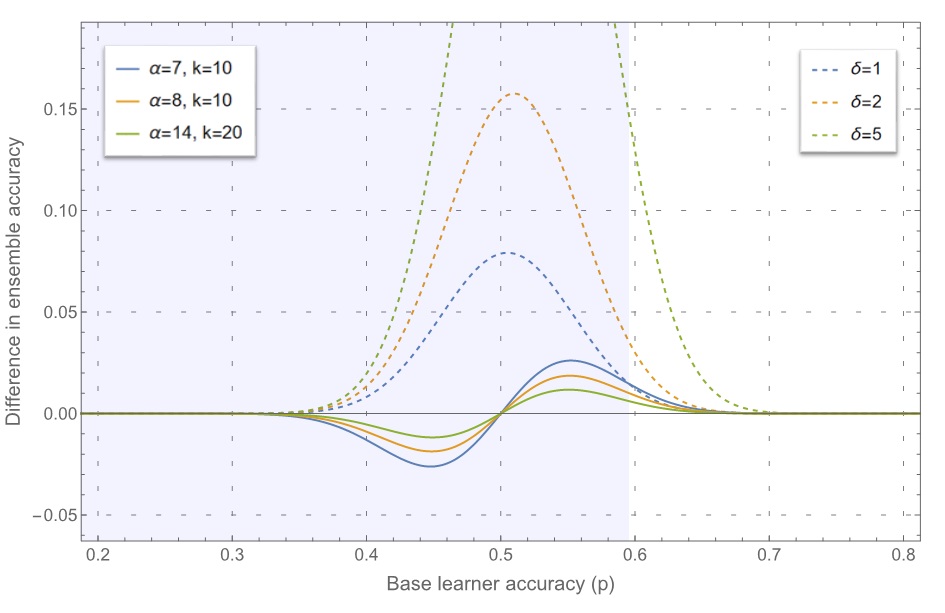}
    \caption{Solid lines: difference in accuracy between the simple majority rule and the homogeneous Slush algorithm, $\Delta\mathbb{P}$, with $n=101$, against the base learner accuracy $p$. Dashed lines: differences in accuracy between majority and $\delta$-supermajority rules. The shaded area shows the region where the Slush algorithm with $\alpha =7$, $k=10$ outperforms the $\delta=1$ supermajority.}\label{fig: Homogeneous Ensembles}
\end{figure}%

\subsubsection{Supermajority rules} \label{sec: supermajority}

The simple majority rule discussed so far requires that more than half of the votes are cast for an output. Supermajority rules increase this acceptance threshold and can lead to enhanced stability in the voting process, as well as increased legitimacy in the final decision~\cite{nitzan1984qualified}.

\begin{definition}[$\delta$-supermajority rule]  \label{def: supermajority rule}
    A $\delta$-supermajority rule is a majority rule for which the acceptance threshold required to choose an alternative is $\ceil*{{n\ov2}} + \delta$. Alternatively, one can use a fraction $q$ of the votes, with $\ceil*{qn}$ votes required for taking a decision.\footnote{This decision rule is not decisive, in the sense of footnote~\ref{footnote: DDR}.}
\end{definition}

The following theorem shows that the Slush algorithm can outperform any $\delta\geq 1$ supermajority rule, even in a homogeneous setting. This is rather noteworthy since low $\delta$ supermajority rules are still very close to being optimal decision rules.

\begin{theorem}\label{thm: Slush vs delta-majority}
    For any $\delta > 0$, there exists a value $p_{\rm th}(\delta) > {1\ov 2}$ such that the accuracy of Slush algorithm built with a homogeneous group of independent base learners with accuracies $p \leq p_{\rm th}(\delta)$ will be larger than the accuracy of a $\delta$-supermajority rule.
\end{theorem}

The proof of this theorem can be found in Appendix~\ref{app: proofs}. To get a grasp on how the threshold value $p_{\rm th}$ changes with $\delta$, we perform a numerical analysis below.

\paragraph{Numerical bounds.} Figure~\ref{fig: Homogeneous Ensembles} shows a comparison between majority aggregation rules and the Slush algorithm, for a homogeneous group of independent learners. We highlight that the Slush algorithm appears to perform exceptionally well compared to supermajority rules, even for $\delta=1$. This indicates that the bound found in Theorem~\ref{thm: Slush vs delta-majority} should increase rapidly with $\delta$. Table~\ref{tab: threshold (honest)} gives some numerical values\footnote{More precisely, the values shown in Table~\ref{tab: threshold (honest)} are lower bounds for $p_{\rm th}$.} for the threshold value $p_{\rm th}$, approximated to two decimal places, for $n = 101$ and a varying $\delta$. These values validate our expectations that $p_{\rm th}$ increases rather fast with $\delta$.
\begin{table}[!ht]
    \centering
    \begin{tabular}{cccccc}
      \toprule
       $\delta$   & 0 & 1 & 2 & 3 & 4 \\ \midrule
    $\{6, 10 \}$  & $0.5$ & $0.56$ & $0.62$ & $0.68$ & $0.73$ \\
      $\{7, 10 \}$  & $0.5$ & $0.6$ & $0.68$ & $0.76$ & $0.83$\\ 
      $\{8, 10 \}$ & $0.5$ & $0.63$ & $0.74$ & $0.83$ & $>0.87$\\ 
      $\{14, 20\} $ & $0.5$ & $0.7$ & $0.84$ & $0.87$ & $>0.88$\\ \bottomrule 
    \end{tabular}
    \caption{Lower bounds on the threshold values $p_{\rm th}(\delta)$ for $n=101$ and varying values of $\{\alpha, k\}$, as indicated in the first column.}
    \label{tab: threshold (honest)}
\end{table}

Remarkably, supermajority rules can still be shown to satisfy jury theorems, as long as the base learner accuracy is larger than the acceptance threshold -- see Theorem~2 of~\cite{fey2003note}. These results allow us to improve on the ``large $n$" behaviour of the Slush algorithm from Proposition~\ref{prop: Slush lower bound}. In the following, we shall use the fraction $q > {1\ov 2}$ of votes required to accept a proposal when discussing a supermajority rule, as introduced in Definition~\ref{def: supermajority rule}.

\begin{lemma}   \label{lemma: Slush convergence}
    For large enough $n$, the accuracy of the Slush algorithm with homogeneous and independent base learners can be brought arbitrarily close to 1, if the base learner accuracy $p$ satisfies
    \be \label{Slush convergence interval}
        q < p < p_{\rm th}(q)~,
    \ee
    for some $q > {1\ov 2}$, where $p_{\rm th}(q)$ is the value below which the Slush algorithm outperforms the $q$-supermajority rule. 
\end{lemma}

This result combines Theorem~\ref{thm: Slush vs delta-majority} with Theorem~2 of~\cite{fey2003note}, and thus the proof is straightforward. The latter claims that the $q$-supermajority rule leads to unit success probability for large $n$ as long as $p > q$. To get a sense of how the interval in Lemma~\ref{lemma: Slush convergence} evolves with $n$, we list approximate values of $p_{\rm th}(q)$ in Table~\ref{tab: threshold for varying n}.
\begin{table}[!ht]
    \centering
    \begin{tabular}{ccccc}
      \toprule
      $n$   & 51 & 101 & 201 & 501 \\\midrule
    $p_{th}(q)$ & 0.92 & $0.88$ & $0.8$ & $0.73$   \\ \bottomrule
    \end{tabular}
    \caption{Lower bound on threshold value $p_{\rm th}(q)$, for $q = 0.55$, $k=10$, $\alpha = 7$ and varying $n$.}
    \label{tab: threshold for varying n}
\end{table}
Of course, it is not obvious that $p_{\rm th}(q)$ remains greater than $q$ as $n$ increases further. Resolving this issue is crucial for proving Theorem~\ref{thm: Slush convergence}.

Lemma~\ref{lemma: Slush convergence} makes a clear statement on the accuracy of the Slush algorithm for $q < p < p_{\rm th}(q)$. We would like to extend this interval further, for any $p > p_{\rm th}(q)$ values. For this, we use the following result.

\begin{lemma}    \label{lemma: Slush monotonicity}
    The accuracy of the Slush algorithm with homogeneous and independent base learners is a strictly monotonically increasing function of the base learner accuracy.
\end{lemma}

\begin{proof}
    Let us define the function $F(b^*, n; p)$ as:
    \be
        F(b^*, n; p) = \sum_{b = b^*}^n \binom{n}{b} p^b (1-p)^{n-b} = 1 - \sum_{b = 0}^{b^*-1} \binom{n}{b} p^b (1-p)^{n-b}~.
    \ee
    The Slush algorithm accuracy defined in \eqref{Slush prob homogeneous} can be also expressed in terms of this function as follows:
    \be \label{Slush prob homogeneous rewritten}
        \mathbb{P}_S(p) =  \bb_{0}  F(0, n; p) + (\bb_1 - \bb_0)  F(1, n; p)  + \ldots + (\bb_n - \bb_{n-1})  F(n, n; p)~.
    \ee
    The veracity of this statement can be checked backwards: from \eqref{Slush prob homogeneous rewritten} we can collect the terms $p^b(1-p)^{n-b}$ for all $b$ and see that \eqref{Slush prob homogeneous} is recovered. Note that the coefficients of all $F(b^*, n; p)$ terms are positive due to Remark~\ref{remark: Slush absorption monotonicity}. As such, if $F(b^*, n; p)$ were an increasing function of $p$, then so would $\mathbb{P}_S(p)$. To show this, we look at the first derivative
    \bea
        {d \ov dp} F(b^*, n; p) & = - \sum_{b=0}^{b^* - 1} \binom{n}{b} p^{b-1}(1-p)^{n-b-1} (b-np) \\
        & = \binom{n}{b^*} b^*\,p^{b^*-1} (1-p)^{n-b^*} > 0~,
    \eea
    which is clearly positive. The proof of the identity used on the second line can be found in Appendix~\ref{app: Avalanche}, Lemma~\ref{lemma: F(p) monotonicity}. This concludes the proof.
\end{proof}

\paragraph{Towards proving Theorem~\ref{thm: Slush convergence}.} Equipped with the above results, we are now ready to sketch the proof for Theorem~\ref{thm: Slush convergence}. The full proof of the theorem can be found in Appendix~\ref{app: proofs}.\\

\begin{SketchProof}
The first statement of the theorem $(\mathbb{P}_S \geq p)$ follows rather simply from an application of Lemma~\ref{lemma: Slush symmetry}. Nevertheless, a proof of the second statement $(\mathbb{P}_S \to 1)$ requires multiple ingredients. This statement builds on Theorem~\ref{thm: Slush vs delta-majority}, according to which large accuracies for the Slush algorithm can occur whenever the base learner accuracy $p$ is in the interval $(q, p_{\rm th}(q))$. Using the monotonicity of the Slush algorithm proved in Lemma~\ref{lemma: Slush monotonicity}, we can eliminate the upper bound of this interval. 
    
However, the only remaining issue is showing that $q < p_{\rm th}(q)$, for any $n$, such that Theorem~\ref{thm: Slush vs delta-majority} is valid. This part of the proof is rather tedious, and can be found in  Appendix~\ref{app: proofs}.
\end{SketchProof}

\subsection{Diversifying the base learners}\label{sec: diversity}

The Slush algorithm combines the distribution of outputs of the base learners with an additional random variable responsible for the communication phase. For a general heterogeneous setting, the binomial distribution from the homogeneous problem is replaced by a Poisson binomial distribution. The ratio $\kappa_b$ introduced in the proof of Theorem~\ref{thm: Slush vs Majority hom} was shown to dictate the behaviour of the Slush algorithm, as compared to the majority rule. For future reference, we define it more formally below.

\begin{definition}[Control ratio]
    Consider a learning problem with $n$ base learners. Let $b$ be the number of blue nodes at the start of the communication phase, and let $\mathbb{P}(S_n= b)$ be the probability that this state can arise from the initial outputs of the base learners, where $S_n = \sum X_i$ is the sum over the Bernoulli random variables assigned to the base learners. Then, we define the control ratio as:
\be \label{def: control ratio}
    \kappa_{b} =  {\rr_{b} \ov \bb_{n-b}} \times {\mathbb{P}(S_n = b) \ov \mathbb{P}(S_n = n-b)}~,
\ee
where $\rr_{b}$ (respectively $\bb_{b}$) are the absorption probabilities in the \emph{all-red} (\emph{all-blue}, respectively) states, starting from a state with $b$ blue nodes.
\end{definition}

This particular definition of the control ratio will be relevant for values $b \geq \ceil*{{n\ov 2}}$. The control ratio can be used to deduce a set of simple sufficient conditions for the Slush algorithm to outperform the majority rule: $\kappa_{b} < 1$. The argument for this statement is very similar to that presented in the proof of Theorem~\ref{thm: Slush vs Majority hom}. As such, a simple analysis of this ratio leads to insights into the type of problems where the Slush algorithm would be more suitable. Schematically, we can interpret the control ratio as
\be \label{control ratio schematically}
    \kappa = \left( \begin{matrix}
        \textit{Asymmetry of} \\
        \textit{consensus protocol}
    \end{matrix} \right) \times \left( \begin{matrix}
        \textit{Diversity of} \\
        \textit{base learners}
    \end{matrix} \right).
\ee
Here, the \emph{diversity} is measured by the variance of the distribution of accuracies, as per Definition~\ref{def: diversity}. Based on these two factors, we have the following two cases of interest:
\begin{itemize}
    \item Asymmetric problems, \ie problems which break the symmetry of the Slush protocol from Lemma~\ref{lemma: Slush symmetry} in favour of the correct output.
    \item Heterogeneous problems, \ie situations in which we deal with a diverse group of learners, which include both strong and weak learners.
\end{itemize}

\noindent In view of these conclusions, we consider a semi-homogeneous setting, with two performance groups. The main result of this subsection is the following theorem.

\begin{theorem} \label{thm: performance groups}
   (Performance groups) Consider a binary classification task with two homogeneous groups of classifiers of sizes $n_1$ and $n_2$ with $n_1> n_2$, having accuracies $p_1$ and $p_2$, respectively. Assume that the classifiers are independent and let $n = n_1 + n_2$, with $n$ odd. Then, for $p_2 < {1\ov 2}$, there exists $p_{\max} > {1\ov 2}$ such that the Slush algorithm outperforms the majority rule for any $p_1$ in the region $[0, p_{\rm max}]$, where $p_{\max}$ is bounded by
    \be
       p_{\rm max} \leq {\ceil*{{n\ov 2}} \ov 1+ n_1}~,
    \ee
    with the upper bound reached for the limiting case $p_2 = 0$.
\end{theorem}

The proof of this statement is left to Appendix~\ref{app: proofs}. Let us stress that in this setting the weak learners will still truthfully follow the modelling process. Thus, they can change their state in the communication phase, according to the majority of the sampled nodes, even if $p_2 = 0$. In Section~\ref{sec: simulation}, we will generalise this setting to a completely heterogeneous one. There, we will see that the Slush algorithm can significantly outperform a centralised majority rule.

\subsection{Byzantine tolerance in consensus learning}  \label{sec: Byzantine}

Byzantine fault-tolerant consensus protocols are known for their ability to tolerate failures and withstand malicious attacks on a network. Accordingly, the next logical problem deserving our attention is that of a network which includes such Byzantine participants, as introduced in Definition~\ref{def: Byzantine}.

Malicious nodes will usually follow an \emph{adversarial strategy}, with the aim of steering the network towards their preferred outcome. Note that this is different from the scenario described by Theorem~\ref{thm: performance groups}, as, in that case, all participants behave the same during the communication phase. To describe the behaviour of the network in the presence of such participants, we can consider the extreme scenario in which the Byzantine participants know the correct outcome with certainty, but decide against sharing it.

\begin{definition}[Perfectly Byzantine participant]  \label{def: perfect Byzantine}
    A perfectly malicious (or perfectly Byzantine) participant is a participant who knows with certainty the correct label of any data inputs, and who, when queried, will always respond with the wrong label.
\end{definition}

Given this definition, we can now consider a learning problem which includes a number of perfectly malicious participants.  Let us point out that the scenario described by the following theorem is still one with a high degree of homogeneity; accordingly, the majority rule is expected to perform rather well. 

\begin{theorem} \label{thm: Perfectly malicious nodes}
    (Perfectly malicious nodes) Consider a group of $f < \alpha$ perfectly malicious participants, in an otherwise homogeneous group of independent base learners of size $n$. Let the accuracy of the $c = n-f$ honest base learners be $p$. Then, the majority rule will outperform the Slush algorithm with parameters $k$ and $\alpha$, as long as $ p > {1\ov2}$ and $\alpha \neq \ceil{{n\ov 2}}$. The Slush algorithm will achieve the same accuracy only for $\alpha = \ceil*{{n\ov 2}}$.
\end{theorem}

This theorem is the analogous of Theorem~\ref{thm: Slush vs Majority hom} to the scenario involving Byzantine nodes. Its proof makes use of the exact form of the absorption probabilities in the Slush protocol in the presence of Byzantine nodes, as described by Lemma~\ref{lemma: Slush asymmetry with Byzantine}, which can be found in Appendix~\ref{app: Avalanche}. 

The observant reader may have noticed the constraint $\alpha > f$ used in the previous theorem. Recall that we model the Slush protocol as a continuous-time Markov chain, with two absorbing states. The death and birth rates are given by the probability that the next query changes a node's colour for the red or blue colours, respectively. If $f$ were to be larger than or equal to $\alpha$, then the state with all honest nodes being blue would no longer be an absorbing state for this process. Thus, to avoid this complication, we set $\alpha > f$ in the statement of the above theorem. Note, however, that in a practical scenario, the protocol will only run a finite amount of rounds. Hence, even in the $\alpha \leq f$ case, there would be a finite probability of reaching consensus for the blue colour. This can be modelled by artificially setting the death rate from the all-blue state to zero.\\ 

\begin{LateProof} \textit{of Theorem~\ref{thm: Perfectly malicious nodes}.}\, The majority rule and Slush algorithm accuracies are given by:
    \bea
    \mathbb{P}_{\rm Maj}(n, c, p) & = \sum_{b = \ceil*{{n\ov 2}}}^c\, \binom{c}{b} p^b (1-p)^{c-b}~, \\
    \mathbb{P}_{\text{S}}(n, c, k,  \alpha, p) & =  \sum_{b=0}^{c}  \binom{c}{b} 
    \bb_b \, p^{b} (1-p)^{c-b}~.
    \eea
From here, we again look at $\Delta \mathbb{P} = \mathbb{P}_{\rm Maj} - \mathbb{P}_{\text{S}}$, and apply the same methods as in the proof of Theorem~\ref{thm: Slush vs Majority hom}. As before, if $\alpha > \ceil*{{n\ov 2}}$, the majority rule outperforms Slush, while the performances are identical for $\alpha =  \ceil*{{n\ov 2}}$. Then, using Lemma~\ref{lemma: Slush asymmetry with Byzantine}, it follows that the control ratio satisfies:
\be
    \kappa_b > {n-b \ov c-b} \times  \left( {p\ov 1- p} \right)^{2b-n} > 1~,
\ee
for  $\ceil*{{n\ov 2}} \leq b \leq n-\alpha$. 

\end{LateProof}

\paragraph{Faulty communication.} Theorem~\ref{thm: Perfectly malicious nodes} offers some insight into how the homogeneous Slush algorithm performs in the presence of Byzantine nodes. In a more realistic scenario, Byzantine behaviour can simply be due to faulty communication. We can expect, in particular, that such participants will only be able to send at most one query, but will not be able to receive any responses. As a result, such participants will effectively drop out in the communication phase. This shows another advantage of consensus learning over centralised ensemble methods since the former can identify this type of faulty participants.

\begin{definition}[Faulty participant]
    A faulty participant is a participant who does not participate in the communication phase.
\end{definition}

\begin{definition}[Perfectly faulty participant]
    A perfectly faulty participant is a faulty participant whose initial output is incorrect.
\end{definition}

We remind the reader that a perfectly faulty participant is not trying to stall the system. Instead, their ML model can be thought of as having very low accuracy (zero) -- see also Section~\ref{sec: modelling assumptions}. The communication phase of the Slush algorithm will be able to detect the faulty participants, which will thus drop out. However, the initial responses of the faulty nodes should still be considered in a majority rule, as there would be no means of identifying a faulty connection in such cases. Thus, introducing perfectly faulty participants is equivalent to considering supermajority rules instead of a simple majority.

\begin{conjecture} 
    Consider a homogeneous group of $c$ independent base learners, each with accuracy $p$. Let there be $f$ perfectly malicious participants and $f'$ perfectly faulty participants. For an appropriate choice of $f' > 0$, there exists a value $p_{\rm th} > {1\ov 2}$ such that the Slush algorithm with parameters $\alpha > f$ and $p \leq p_{\rm th}$ outperforms the majority rule.
\end{conjecture}

As previously alluded to, it is not difficult to see that this statement reduces to a comparison between the Slush algorithm and $\delta$-supermajority rules, with $\delta = f'$. As such, this proposition is an interpolation between Theorem~\ref{thm: Perfectly malicious nodes} and Theorem~\ref{thm: Slush vs delta-majority}. We do not have a proof of this statement, but we offer some numerical evidence below.

Remarkably, small values of $\delta = f'$ already ensure that the threshold value $p_{\rm th}$ is larger than 50\%. Table~\ref{tab: threshold with Byz nodes} gives some numerical values for $p_{\rm th}(\delta)$, approximated to two decimal places, for  $n = 101$, and fixed protocol parameters.
\begin{table}[!ht]
    \centering
    \begin{tabular}{ccccc}
    \toprule
       $\delta$   & 0 & 1 & 2 & 3  \\ \midrule 
       $f = 0$ & 0.5 & 0.7 & 0.84 & 0.87 \\ 
       $f = 1$ & 0.49 & 0.69 & 0.84 & 0.87  \\
       $ f = 5$ & 0.47 & 0.67 & 0.83 & 0.87 \\
       $f = 10$ & 0.47 & 0.67 & 0.83 & 0.87 \\ \bottomrule
    \end{tabular}
    \caption{Lower bounds on threshold value $p_{\rm th}(\delta)$ for fixed $n=101$, $k=20$, $\alpha=14$ and varying $f$ and $\delta$.}
    \label{tab: threshold with Byz nodes}
\end{table}

\section{Numerical simulations}   \label{sec: simulation}

In this section, we present numerical simulations of the Slush consensus learning algorithm that extend beyond our previous analysis. It is worth pointing out that the communication phase in the previously described scenarios did not make use of the conviction of the base learners in their prediction. In this regard, we will also present a modified Slush algorithm which uses local parameters instead of globally defined ones. Such modifications can stir the network in favour of the better classifiers.

\begin{definition}[Strong confidence]  \label{def: strong confidence}
    A participant whose local model has accuracy $p > {1\ov 2}$ is said to have strong confidence in their result if their local threshold parameter for accepting a query satisfies $\alpha \geq k \,p$. 
\end{definition}

To simulate the behaviour of a realistic fully decentralised network, we will consider non-iid data. It should be noted that heterogeneity among the training sets has been shown to pose serious challenges for FL algorithms to achieve high levels of precision~\cite{https://doi.org/10.48550/arxiv.1806.00582}. In Section~\ref{sec: MNIST} we focus on the FEMNIST dataset~\cite{caldas2019leaf}, which is a dataset designed for non-iid federated learning. Section~\ref{sec: beta distribution} presents a generalisation of this setting. More details about our simulations are discussed in Appendix~\ref{app: simulation details}.

\subsection{Non-IID MNIST dataset} \label{sec: MNIST}

Realistic datasets for fully decentralised distributed learning are typically proprietary and not available to the public. A modular benchmarking framework for federated learning is provided by the LEAF bechmark~\cite{caldas2019leaf}, which organises well-established datasets for realistic distributed setting applications. An example is the FEMNIST dataset, which partitions the extended MNIST dataset~\cite{lecun1998mnist, cohen2017emnist} by the writer of the digit or character into 3550 non-iid sets. The individual sets have, on average, around 227 samples (with a standard deviation of 89 samples)~\cite{caldas2019leaf}.

Due to the small number of samples, we will only train simple models for the data of a single user. Thus, we will run two types of simulations on the FEMNIST dataset. First, we build 101 different models, each being trained on the data of a single different user. Second, we will group users together and train more intricate models on the grouped datasets.

\begin{figure}[!ht] %
    \centering
    \includegraphics[width=0.98\textwidth]{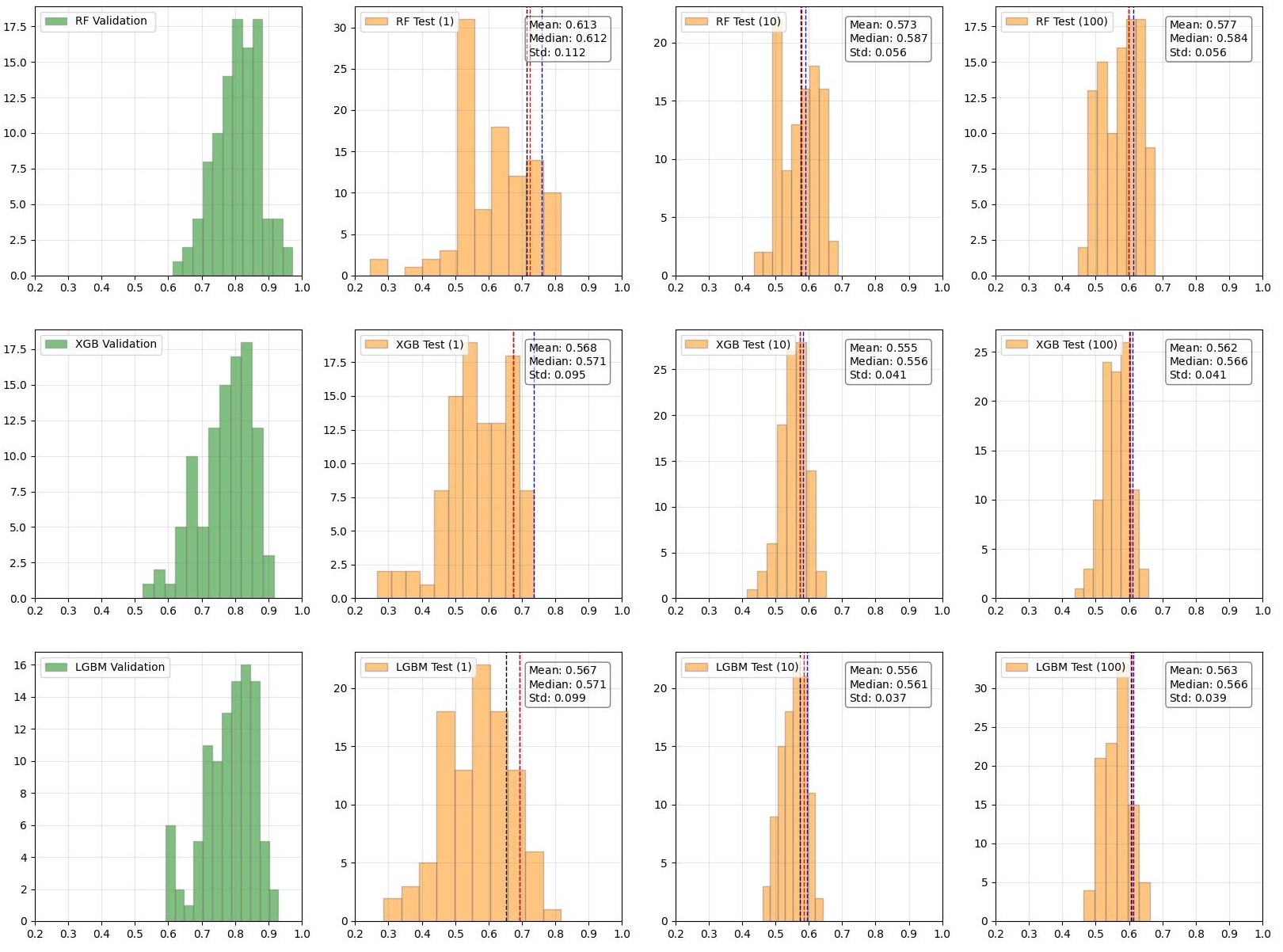}
    \caption{Simulation of Slush consensus learning on FEMNIST dataset for $n=101$ base learners. \emph{Green:} Distribution of accuracies on validation sets. \emph{Orange:} Distribution of accuracies on test sets using data from 1, 10 or 100 new users, respectively. The dashed lines correspond to: majority ensemble (black), Slush with $\alpha = 6$ (red) and Slush with local $\alpha$ parameters, \ie strong confidence (blue).}\label{fig: feMNIST letters niid}
\end{figure}%

\subsubsection{Models for individual users}

To begin with, we pick $n = 101$ sub-datasets of the FEMNIST dataset. Each such dataset contains data for 62 distinct classes: 10 digits, 26 lower case letters and 26 upper case letters. We create a binary classification task by labelling all lower case letters by 0, and the upper case letters by 1, while discarding the digits.

The next step is to train $101$ different models on these datasets. We deploy three different types of models: random forests (RF)~\cite{breiman2001random}, extreme gradient boosting (XGBoost)~\cite{chen2015xgboost} and light gradient-boosting machine (LGBM)~\cite{ke2017lightgbm}. Each model is trained on 80\% of their own data, with the performance on the remaining 20\% of the data recorded on the left column of Figure~\ref{fig: feMNIST letters niid}. As before, we use accuracy as the performance metric for evaluating the models. Importantly, all models have accuracies better than 50\% on their respective validation sets. 

The performance of these models can be, of course, slightly improved, by tuning the model parameters using a grid search, for example. Random forests tend to be rather robust to overfitting since the decision trees assembling the forests are independent. As such, the default parameters do typically lead to reasonably well-trained ensembles. While gradient boosters can improve the accuracy of random forests, they can lead to overfitting since they repeatedly fit new models to the residuals of previous models. These features can be also seen in the left columns of Figure~\ref{fig: feMNIST letters niid}. Nevertheless, these simple models should suffice for our purposes. Note that some of the low accuracies on the test sets can be largely attributed to the small size of the training sets.

\begin{figure}[!ht] %
    \centering
    \includegraphics[width=0.98\textwidth]{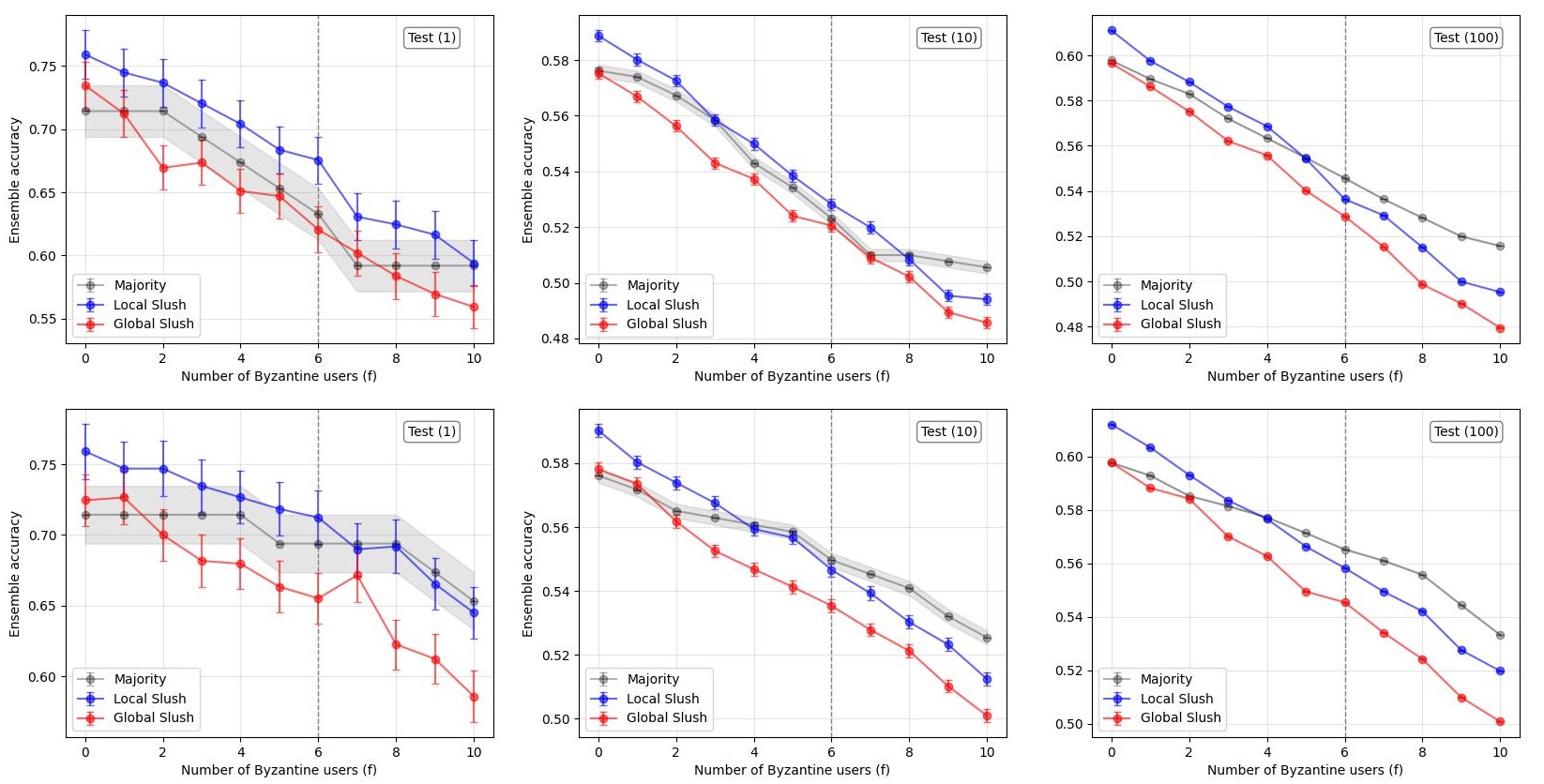}
    \caption{Accuracies of ensembles built from 101 Random Forest models against number of perfectly malicious models. The honest models are trained on non-iid samples from the FEMNIST dataset and tested on data from 1, 10 and 100 new users, respectively. The dashed line $f = 6$ is the value of the $\alpha$ parameter used for the \emph{global} Slush algorithm. \emph{Top row:} strong classifiers turn Byzantine. \emph{Bottom row:} weak classifiers turn Byzantine.}\label{fig: feMNIST letters niid Byzantine}
\end{figure}%

For testing, we use three different datasets, from 1, 10 and 100 combined new users. These sets consist of 49, 453 and 4741 samples. The distributions of accuracies in these three testing scenarios are shown in orange in Figure~\ref{fig: feMNIST letters niid}, for the three types of models built. It should be noted straightaway that many of the models do not generalise very well to the unseen handwriting. For each of the test datasets, we also generate the majority ensemble, as well as consensus learning ensembles. This will no longer be the case when we increase the size of the training sets in the next subsection.

For the communication phase of the consensus learning algorithms, we run a Slush protocol with a global threshold parameter $\alpha = 6$, and one with local parameters determined by the accuracy on the test sets, in accordance with the concept of strong confidence introduced in Definition~\ref{def: strong confidence}.\footnote{In practical terms, only part of the test subset should be used for determining these local parameters. However, since some of the test sets are rather small, we make an implicit assumption that the test set accuracy would be the same as the accuracy on a smaller subset of this set. This only assumes that this subset is identically distributed to the whole test set.} In both cases, the protocols use $k=10$ for the sampling of participants. For speeding up the computation time, each communication phase only lasts for 50 rounds per node (so a total of 5000 rounds), which should be more than enough to ensure convergence~\cite{rocket2020scalable, amoressesar2024analysis}. The results are indicated with dashed lines in Figure~\ref{fig: feMNIST letters niid}.

In Figure~\ref{fig: feMNIST letters niid Byzantine} we also consider Byzantine participants. For this, $f$ base learners are selected at random from the $101$ models and are turned into perfect Byzantine models, as per Definition~\ref{def: perfect Byzantine}. We consider two different samples of base learners that are turned into perfectly Byzantine models: strong classifiers and weak classifiers. More details about the performance of these users on the three test sets can be found in Appendix~\ref{app: simulation details}.

Figure~\ref{fig: feMNIST letters niid Byzantine} shows how the accuracies of the ensembles vary when the number of Byzantine participants increases. Let us note that in the context of Byzantine fault tolerance, a centralised majority rule would inevitably have increased Byzantine resilience compared to any consensus protocol. The maximum number of Byzantine participants in the former case is 1/2 of the total network participants, as opposed to the classical 1/3 value for consensus protocols~\cite{10.1145/322186.322188}. This argument can be extended to consensus learning algorithms. Regardless, consensus learning appears to outperform the majority rule even in the presence of Byzantine users, as long as their number is not too large. 
The simulations also indicate that when strong models are converted to Byzantine models, the local Slush algorithm has increased resilience as compared to the majority rule.

\begin{figure}[t]
    \centering
    \begin{subfigure}{\textwidth}
    \centering
    \includegraphics[width=0.98\textwidth]{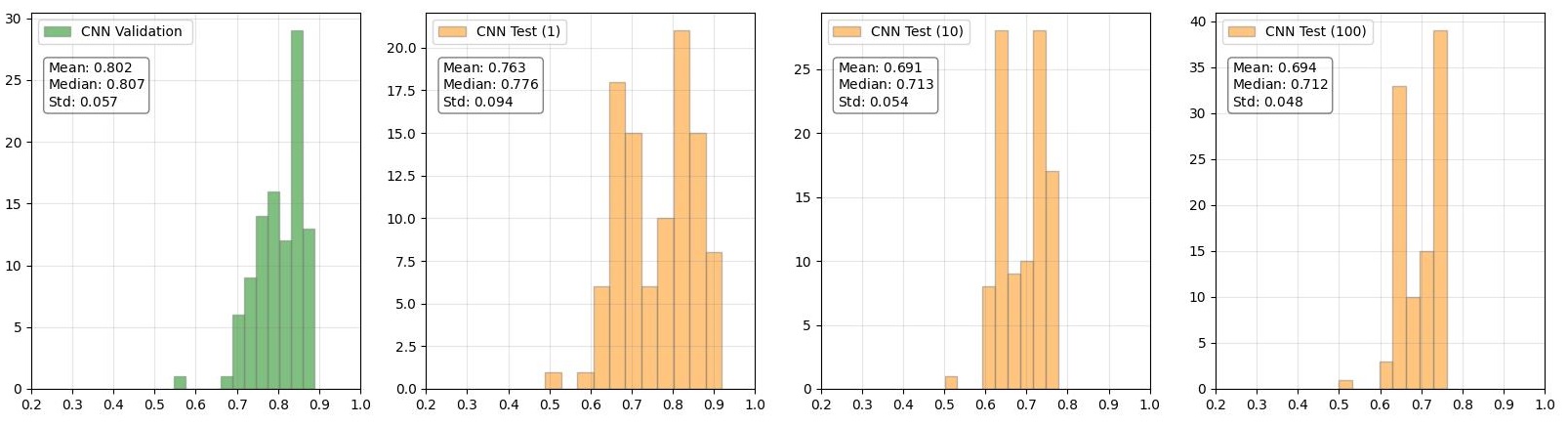}
    \end{subfigure}
    \hfill
    \begin{subfigure}{\textwidth}
    \centering
    \includegraphics[width=0.98\textwidth]{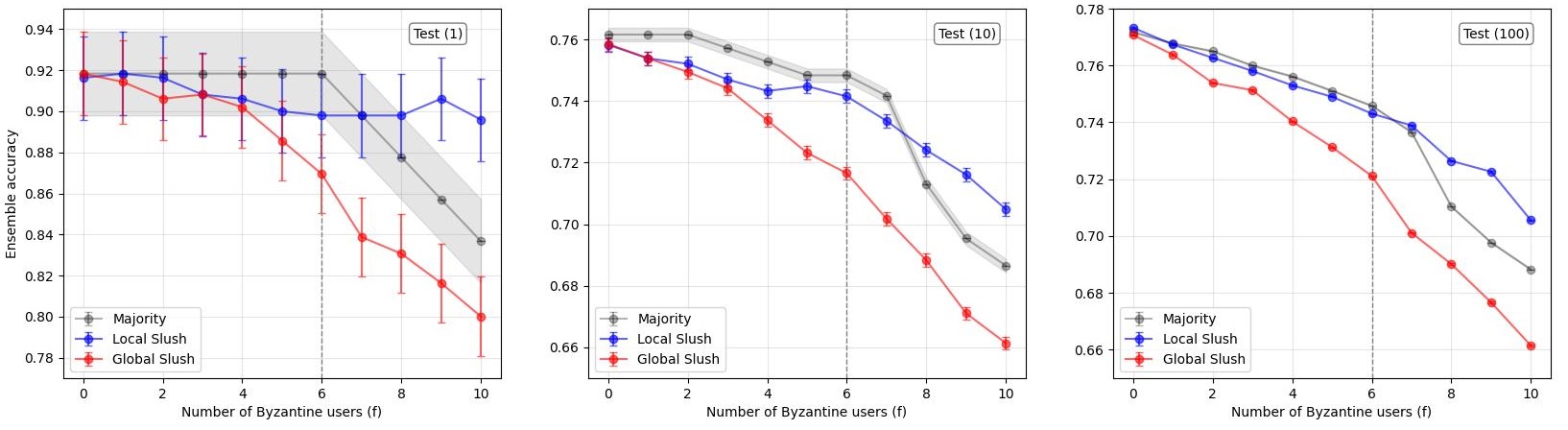}
    \end{subfigure}%
    \caption{\emph{Top:} Distribution of accuracies on validation sets (green) and test sets (orange) for the 101 CNNs. \emph{Bottom:} Accuracies of ensembles of 101 CNNs against number of perfectly malicious nodes for the three test sets.}
    \label{fig: feMNIST letters niid CNN}
\end{figure}

\subsubsection{Models for groups of users}

Convolutional neural networks (CNNs) are known to be particularly suitable for image recognition. In this subsection, we train such models on enlarged datasets and subsequently create ensembles of CNNs. To compare the results with the results from the previous section, we will stick to $n = 101$ distinct models.  We will use the same test sets as in the previous simulations, which are separate from the training and validation sets. 

To avoid overfitting, we use rather simple CNNs, with four hidden layers: a convolutional layer with 8 filters of size $(3,3)$; a max pooling layer of size $(2,2)$; a flattening layer; a dense layer of 100 neurons with ReLU activation functions and He weight initialisation scheme \cite{Goodfellow-et-al-2016}. For the stochastic gradient descent optimiser, we set the learning rate to 0.01 and use a momentum of 0.9. As the training sets are not too large, we do not use any cross-validation. The models are trained for 10 epochs, with batch sizes of 32 examples. 

Perhaps not surprisingly, these CNNs generalise much better than the models trained on data coming from a single user. This is reflected in the distribution of accuracies in the top row of Figure~\ref{fig: feMNIST letters niid CNN}. The bottom row in the figure also shows how the ensembles built from these 101 CNNs behave once base learners turn (perfectly) Byzantine. In this case, the sample of 10 base learners that are turned into Byzantine models is chosen to be representative of the whole ensemble of 101 models. More details about this sample are again left to Appendix~\ref{app: simulation details}. Rather remarkably, we see that the local Slush algorithm still has increased robustness against Byzantine participants and can outperform a centralised majority rule even for larger numbers of Byzantine participants.

\subsection{Beta-distributed base learners} \label{sec: beta distribution}

The previous simulations on the FEMNIST dataset strengthen our position that consensus learning algorithms can perform better than centralised majority rules if the base learners are diverse. This situation is highly probable within a realistic context and could occur whenever the base learners are trained on non-iid data. These insights are in perfect agreement with our analysis of the control ratio in \eqref{control ratio schematically}.

In this section, we present a generalisation of these results, by generating samples of accuracies for independent base learners from a beta distribution. The beta distribution, ${\rm B}(a,b)$, is a bounded continuous probability distribution characterised by two real shape parameters, $a, b$, with $a, b> 0$.\footnote{Here, $b$ should not be confused with the number of blue nodes in the Slush protocol.} The distribution approximates a uniform distribution when $ a = b = 1$, and a normal distribution for large $a, b$, with $a \approx b$. 

\begin{figure}[!ht] %
    \centering
    \includegraphics[width=0.98\textwidth]{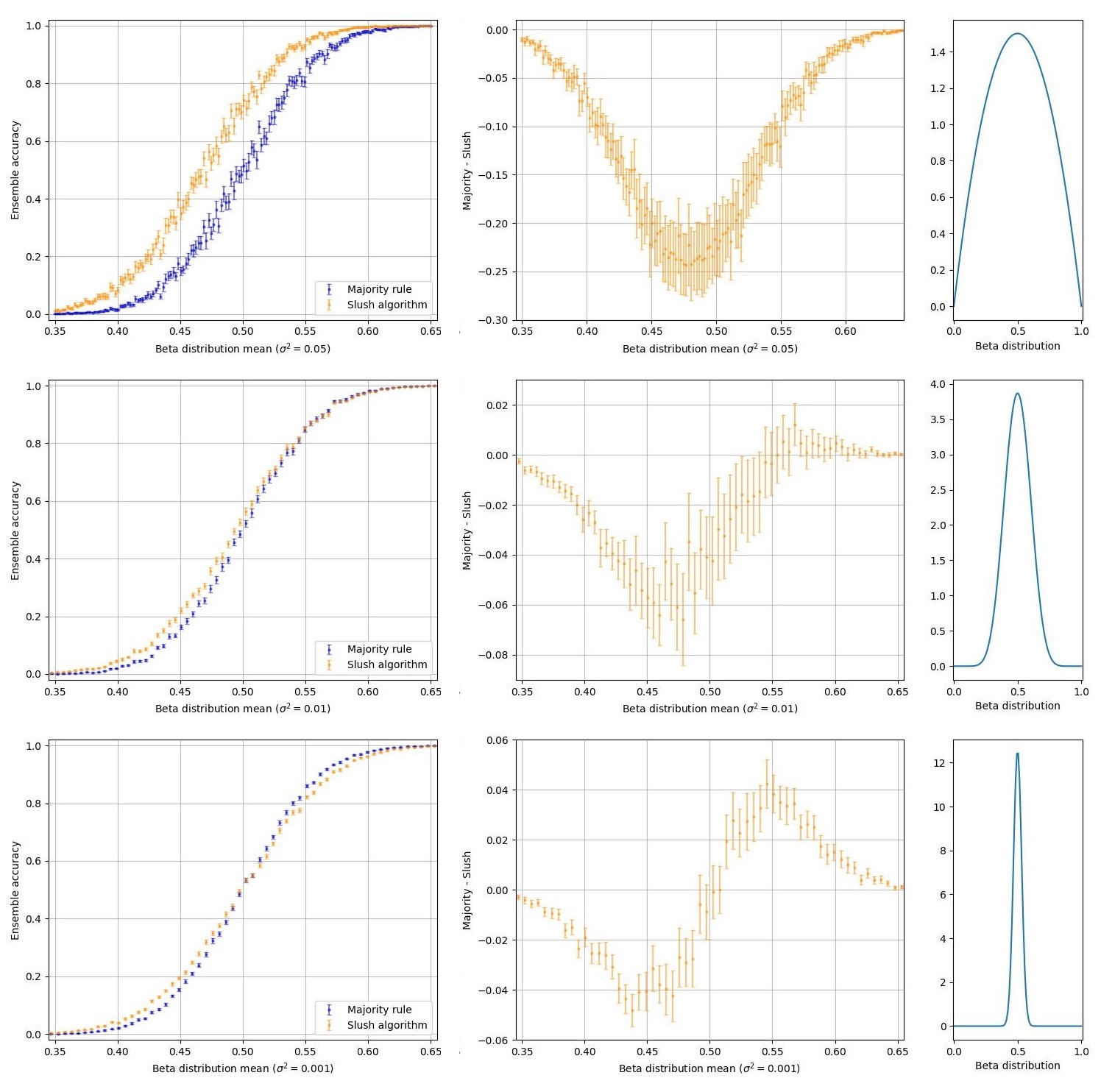}
    \caption{Simulation of Slush consensus algorithm for $n=101$, $k=10$, $\cN_1 = 100$, $\cN_2 = 50$ and local $\alpha$ parameters. \emph{Left column:} Ensemble accuracy for varying mean and fixed variance of the beta distribution. \emph{Middle column:} Difference in accuracy between majority rule and Slush algorithm, for fixed variance. \emph{Right column:} Probability density function of beta distribution with mean equal to $0.5$ and variance representative of the row.}\label{fig: Simulation}
\end{figure}%

It is, of course, difficult to estimate the distribution of accuracies of the base learners without more specifications about the difficulty of the task, or details about the training data used by the base learners. Independent and identically distributed sampling for the training datasets is more likely to lead to distributions that are close to a normal distribution -- see \eg \cite{8857234}. On the other hand, if there are reasonable expectations that the base learners can achieve high accuracies, then a folded normal distribution would be more appropriate.\footnote{However, in such scenarios, ensemble methods might not be necessary at all for improving accuracy.} Regardless, our previous simulations do appear to suggest that the beta distribution is a valid candidate.

The simulation can be described by the following steps:
\begin{enumerate}
    \item A sample of base learner accuracies of size $n$ is generated from a beta distribution with fixed mean and variance.
    \item For each sample, we generate a vector of individual outputs (\ie a voting profile) and simulate the output of the Slush protocol as well as that of a majority rule using this voting profile. The process is repeated $\cN_1$ for a given sample of accuracies, and the accuracy of the ensemble is determined as the percentage of correctly identified outputs.
    \item Step 2 repeats $\cN_2$ times for a given beta distribution, to eliminate any statistical outliers.
\end{enumerate}

\noindent The error in the ensemble accuracy is measured as the sampling error for the set of $\cN_2$ values obtained in step~3 of the simulation. The exact value of the error can be determined as discussed in Appendix~\ref{app: simulation details}, but this would significantly increase the computation time. Finally, the process is repeated for different means and variances of the beta distribution, which is always taken to have a concave density function.

The results of the simulations are shown in Figure~\ref{fig: Simulation}. As before, we use a Slush algorithm with local threshold parameters $\alpha$. As expected from the previous theoretical analysis, the Slush algorithm performs significantly better compared to a centralised majority rule for larger variances in the distribution of accuracies.

\section{Conclusions and outlook} \label{sec: conclusions}

In this work, we introduced a novel distributed ML paradigm -- consensus learning -- readily described as a fully decentralised ensemble method that deploys a consensus protocol. We analysed how a typical consensus learning algorithm built on a probabilistic consensus protocol behaves in different scenarios and what improvements it brings to established ML methods. Consensus learning has clear advantages over other distributed learning algorithms, which include communication overhead, resilience against malicious users and protection of private data. Moreover, consensus learning preserves model explainability of typical ensemble weighting methods, offering a high degree of transparency and interpretability.

Our concrete results offer lower bounds on the accuracy of consensus learning classifiers using the Slush consensus protocol, while also describing the behaviour for a large enough number of base learners. We stress again that the proofs of our main results only use \emph{weak} features of the Slush consensus protocol, and could thus be generalised to other probabilistic protocols. In addition to these results, our numerical analysis shows that a relatively simple modification of the Slush consensus protocol can lead to increased Byzantine resilience and a boost in performance.

Our analysis also indicates that a greedy consensus which favours high-accuracy nodes might perform better than Slush algorithms for an ML task. For this purpose, it would be interesting to study federated byzantine agreement (FBA) consensus protocols~\cite{mazieres2021stellar} where each node chooses which participants to trust, as well as hierarchical consensus protocols, which delegate leaders for each communication round. Note, however, that such protocols do typically require precise knowledge of the network, through the so-called \emph{quorum} membership, in order to satisfy safety and liveness guarantees. 

Performance boosts can be also achieved by using different local aggregation rules in the communication phase. For a classification task, for instance, these could be weighted majority rules, where, ideally, the weights would be determined by the accuracy of the learners. Such information is typically either not available or, not trustworthy in a distributed setting, but could be stored as some form of past performance in a blockchain implementation. As such, this information would become available to all base learners. Additionally, in such a blockchain implementation, reward mechanisms can create further incentives for participants to be honest, and thus increase the performance of the ensemble. Other Byzantine-resilient aggregation methods (such as those discussed in \eg \cite{10.1145/3616537}) could lead to increased robustness against malicious attacks.

Lastly, consensus learning algorithms can also be applied to other types of ML problems. For regression problems, robust local aggregation rules need to be deployed, similar to Byzantine ML algorithms. The algorithms can also be used for unsupervised ML, similar to other unsupervised ensemble learning methods. We leave a more detailed exposition of these methods for future work.




\section*{Acknowledgements}

We are very grateful to Charles Grover, Sabrina Guettes and Jernej Tonejc for insightful discussions, constructive feedback and valuable input in many of the proofs presented in this work. We would also like to thank Anne Marie Jackson for helpful comments on the final draft.

\appendix
\section{Snow protocols} \label{app: Avalanche}

The Snow family of consensus protocols builds upon the Slush protocol, whose technical aspects we describe in this appendix. 

\subsection{Slush with honest participants} \label{app: Slush honest}

The Slush protocol is fully described by two control parameters $k \in (0,n]$ and $\alpha \in (k/2, k]$; the former is the size of the sample selected by a node for sending a query, while the latter is the threshold parameter for accepting a query. The protocol can be modelled as a continuous-time Markov chain, with the state $\cS$ of the system corresponding to the number of blue nodes (\ie nodes that correctly labelled the transactions) at a given time. This is depicted diagrammatically below:
\bea    
 \begin{tikzpicture}[baseline=1mm,baseline=1mm] 
\node[style={rectangle,draw},thick] (0) []{$0$};
\node[style={rectangle,draw},thick] (1) [right =1.5 of 0]{$1$};
\node[style={rectangle,draw},thick] (2) [right =1.5 of 1]{$2$};

\node[style={rectangle,draw},thick] (n-2) [right =2 of 2]{$n-2$};
\node[style={rectangle,draw},thick] (n-1) [right =1.5 of n-2]{$n-1$};
\node[style={rectangle,draw},thick] (n) [right =1.5 of n-1]{$n$};

\draw[->] (1) to node [midway,above] {$\mu_1$} (0);
\draw[->] (2) to [out=150,in=30] (1); \draw[->] (1) to [out=330,in=210] (2);
\draw[->] (n-1) to [out=160,in=20] (n-2); \draw[->] (n-2) to [out=340,in=200] (n-1);
\draw[->] (n-1) to node [midway,below] {$\lambda_{n-1}$} (n);

\node (int1) [right = 0.35 of 2] {$\ldots\ldots$}; 
\node (mu2) [above right = 0.15 and 0.45 of 1] {$\mu_2$}; \node (l1) [below right = 0.15 and 0.45 of 1] {$\lambda_1$};
\node (mu2) [above right = 0.15 and 0.3 of n-2] {$\mu_{n-1}$}; \node (l1) [below right = 0.15 and 0.3 of n-2] {$\lambda_{n-2}$};
\end{tikzpicture}
\eea
The process has two absorbing states, all-blue and all-red, respectively. The probability that a query for the blue colour reaches the threshold of $\alpha$ or more votes given $b$ blue nodes in the network can be found from a simple combinatorial exercise, being determined by a hypergeometric distribution,
\be\label{hypergeometric sum}
    {\rm H}(n,b,k,\alpha)\equiv \sum_{j=\alpha}^k \binom{b}{j} \binom{n-b}{k-j} \binom{n}{k}^{-1}~.
\ee
An important identity for the normalisation of the hypergeometric distribution is
\be\label{hypergeometric distr normalisation}
    \sum_{j=0}^k \binom{n_1}{j}\binom{n_2}{ k-j} = \binom{n_1 + n_2}{k}~,
\ee
for $n_1, n_2 \in \bN$, with $n_1, n_2 \geq k$. Consider the case where all nodes in the network are honest. Then, the Markov chain is described by the following transition rates.

\begin{definition}[Transition rates]
    The death $\mu_b$ and birth $\lambda_b$ rates for the state with $b$ blue nodes of the Slush protocol with $n$ nodes and protocol parameters $k$ and $\alpha$ are defined as
    \be \label{Slush death-birth rates}
    \mu_b = b\, {\rm H}(n,n-b,k,\alpha)~, \qquad \lambda_b = (n-b)\, {\rm H}(n,b,k,\alpha)~.
    \ee
\end{definition}

The death rate is given by the probability that a given query reaches red consensus; thus, we want one of the $b$ blue nodes to change colour to red. Similarly, for the birth rate, we need one of the $(n-b)$ red nodes to change their colour to blue. These rates satisfy the obvious property $\lambda_b = \mu_{n-b}$. Let us also point out that
\be \label{vanishing rates}
    \lambda_{\,b\,<\,\alpha} = 0 = \mu_{\,b\,>\,n-\alpha}~,
\ee
as the voting threshold cannot be reached by the minorities in these cases.
\begin{theorem}[Slush absorption,~\cite{rocket2020scalable}]\label{thm: exact xi}
    Let the configuration of the system at time $t$ be $\cS$, with $b$ blue nodes, where $0 \leq b \leq n$, and $n-b$ red nodes. Then, the probability of absorption $\rr_b$ in the all-red state is given by
\be \label{exact xi}
    \rr_b = \frac{\sum\limits_{l=1}^{n-b} \prod\limits_{i=1}^{n-l} \mu_i \prod\limits_{j=n-l+1}^{n-1}\lambda_j }{\sum\limits_{l=1}^n \prod\limits_{i=1}^{n-l} \mu_i \prod\limits_{j=n-l+1}^{n-1}\lambda_j}~.
\ee
\end{theorem}
\begin{proof}
    This is a standard death-birth Markov process, which makes use of the steady-state Kolmogorov equations:
    \be
    (\mu_b + \lambda_b)\, \rr_b = \lambda_b\, \rr_{b+1} + \mu_b\, \rr_{b-1}~.
    \ee
    This recursion relation can be solved explicitly for appropriate boundary conditions, \ie $\rr_{b = 0} = 1$ and $\rr_{b = n} = 0$. See \eg \cite{tan1976absorption}, Chapter IV of~\cite{karlin1957classification}, or Theorem~2 of~\cite{rocket2020scalable} for explicit proofs.
\end{proof}

\begin{corollary}   \label{corr: exact zeta}
    The absorption probability $\bb_b$ in the all-blue state given $b$ blue nodes reads
    \be \label{exact zeta}
    \bb_b = \frac{\sum\limits_{l=1}^b  \prod\limits_{i=1}^{l-1}\mu_i \prod\limits_{j=l}^{n-1} \lambda_j }{\sum\limits_{l=1}^n \prod\limits_{i=1}^{n-l} \mu_i \prod\limits_{j=n-l+1}^{n-1}\lambda_j}~.
    \ee
\end{corollary}
\begin{proof}
    Since the Markov process has only two absorbing states, we have $\bb_b + \rr_b = 1$, with $\rr_b$ as given in Theorem~\ref{thm: exact xi}. Using the identity $\lambda_b = \mu_{n-b}$, the result follows.
\end{proof}

From \eqref{vanishing rates}, we also note that $\bb_b = 0$ for $b<\alpha$. Moreover, as $\lambda_{b\,<\,\alpha} = 0$, the sum in the numerator of $\bb_b$ in \eqref{exact zeta} will only contribute with terms starting from $l = \alpha$, to $ l = b$. Similarly, we notice that $\bb_b = 1$ when $b > n-\alpha$, as the absorption probability in the all-red state vanishes in these cases. In practice, however, it might be slightly difficult to work with these expressions.\\

\begin{LateProof}
    \textit{of Lemma~\ref{lemma: Slush symmetry}.} An intuitive argument for the first result follows from the neutrality of the protocol, as defined earlier in the context of decisive decision rules: Slush does not discriminate against one of the absorbing states on labelling grounds. The result can also be proved explicitly from the exact expressions of the absorption probabilities, using $\lambda_b = \mu_{n-b}$.

    Since $B_{b \,> \,{n\ov 2}} > {1\ov 2}$, and $B_{b\,> \, n-\alpha} = 1$, the second part of the theorem is equivalent to the statement that $\bb_b$ is discrete concave for $n-\alpha > b \geq \ceil*{{n\ov 2}}$. Thus, we need to show that
    \be
    \bb_{b-1} + \bb_{b+1} < 2\bb_b~,
    \ee
    which, upon using Corollary~\ref{corr: exact zeta}, amounts to showing that
    \be
    \mu_{b} < \lambda_{b}~,
    \ee
    for the above range for $b$. This can be shown using the explicit form \eqref{Slush death-birth rates} of the transition rates. A sufficient condition is given as follows:
    \be
    b\, \binom{n-b}{j}\binom{b}{k-j} < (n-b) \binom{b}{j}\binom{n-b}{k-j}~,
    \ee
    with $\alpha \leq j \leq k$. Using $b = m+1 + \delta$, with $0\leq \delta \leq m$, for $n = 2m+1$, the above reduces to
    \be
        \prod_{t=k-j}^{j-1} \left(1 - {2\delta + 1 \ov m+1+\delta-t} \right) < 1-{2\delta + 1 \ov m+1+\delta}~,
    \ee
    which is clearly true. 
    
    Finally, when $\alpha  = k = 1$, it follows that $\lambda_b = \mu_b$ for any $1 \leq b \leq n-1$. Then, it is rather straightforward to see that all terms in the summand of \eqref{exact zeta} are equal, and thus $B_b = {b\ov n}$.  This concludes the proof.
    
\end{LateProof}

Another relevant result used in the proof of Lemma~\ref{lemma: Slush monotonicity} is as follows.

\begin{lemma}   \label{lemma: F(p) monotonicity}
    Let $F(b^*, n; p)$ be defined as in Lemma~\ref{lemma: Slush monotonicity}, \ie
    \be
        F(b^*, n; p) = \sum_{b = b^*}^n \binom{n}{b} p^b (1-p)^{n-b} = 1 - \sum_{b = 0}^{b^*-1} \binom{n}{b} p^b (1-p)^{n-b}~,
    \ee
    for $0 \leq b^* \leq n$. Then,
    \bea
        {d \ov dp} F(b^*, n; p)  = \binom{n}{b^*} b^*\,p^{b^*-1} (1-p)^{n-b^*}~.\label{F derivative}
    \eea
    
\end{lemma}

\begin{proof}
    We prove this statement by induction, as follows. First note that ${d \ov dp}F(0, n; p) = 0$ and ${d \ov dp}F(1, n; p) = n(1-p)^{n-1}$, in agreement with the above formula. Then, assuming \eqref{F derivative} holds for ${d \ov dp}F(j, n; p)$, we can find ${d \ov dp}F(j+1, n; p)$ as follows:
    \bea
    {d \ov dp} F(j+1, n; p) & = - \sum_{b=0}^{j} \binom{n}{b} p^{b-1}(1-p)^{n-b-1} (b-np)\\
    & = {d \ov dp}F(j, n; p) - \binom{n}{j}p^{j-1} (1-p)^{n-j-1} (j-np) \\
    & = \binom{n}{j}p^{j-1}(1-p)^{n-j-1} \Big( j(1-p) - (j-np) \Big) \\
    & = \binom{n}{j}(n-j) p^j (1-p)^{n-j-1} = \binom{n}{j+1}(j+1) p^j (1-p)^{n-j-1}~,
    \eea
    which concludes the proof by induction.
\end{proof}

\begin{lemma}   \label{lemma: (b-1)(b+1)<b^2}
    For the Slush protocol with no malicious nodes, the absorption probabilities in the all-blue states satisfy:
    \be
        \bb_{b-1} \bb_{b+1} \leq \bb_b^{\,2}~,
    \ee
    for any $b$ in the range $b\geq \ceil*{{n\ov 2}} $. Furthermore, equality occurs only for $b > n-\alpha$.
\end{lemma}
\begin{proof}
    Since $b \geq \ceil*{{n\ov 2}}$, from the proof of Lemma~~\ref{lemma: Slush symmetry} we have ${1\ov 2}\left(\bb_{b-1} + \bb_{b+1} \right)< \bb_b$. By the AM-GM inequality,\footnote{Arithmetic and geometric means inequality.} we also have:
    \be
        {1\ov 2}\left(\bb_{b-1} + \bb_{b+1} \right) \geq \sqrt{\bb_{b-1} \bb_{b+1}}~.
    \ee
    Combining these two yields the required identity. 
\end{proof}

Numerical simulations indicate that this lemma extends to any $1 \leq b \leq n-1$, with equality for $b < \alpha$ and $b > n-\alpha$. However, in the $b < \ceil*{{n\ov 2}}$ range, we now have $\mu_b > \lambda_b$, and thus ${1\ov 2}\left(\bb_{b-1} + \bb_{b+1} \right) > \bb_b$. As such, the argument presented in the above proof no longer applies. Thus, the proof appears to be more intricate for this range, and we leave it for future work. If proved, this extended version of Lemma~\ref{lemma: (b-1)(b+1)<b^2} would significantly simplify the proof of Theorem~\ref{thm: Slush convergence}.

\subsection{Byzantine participants}

Consider now the scenario described by Theorem~\ref{thm: Perfectly malicious nodes}, \ie there are $f$ Byzantine nodes which follow an ideal adversarial strategy and always respond to a query with the red colour. Let $c = n-f$ be the number of honest participants. In the presence of malicious nodes, the death and birth ratios of the Slush protocol change as follows~\cite{rocket2020scalable}:
\be
     \mu_b = b\, {\rm H}(n, n-b, k, \alpha)~, \qquad \lambda_b = (c-b)\, {\rm H}(n,b,k,\alpha)~,
\ee
for $1 \leq b \leq c-1$, as $\{0\}$ and $\{c\}$ are the absorbing states, where we also assume $c > f$.  Importantly, we have $\mu_{n-b} = {n-b \ov c-b} \lambda_{b}$, while it is still true that $\lambda_b = 0$ for $b<\alpha$, and $\mu_b = 0$ for $b > n-\alpha$. The absorption probabilities $\bb_b$, $\rr_b$ in the all-blue and all-red states, respectively, can be found as before, with the distinction that $n$ is replaced by $c$ in \eqref{exact zeta} and \eqref{exact xi}, since the Markov process has $c+1$ states instead of $n+1$. \\

A new and important result in our analysis is the following lemma.

\begin{lemma}   \label{lemma: Slush asymmetry with Byzantine}
    In the Slush protocol with parameters $k,\alpha$ and $n$ nodes, out of which $0 < f< \alpha$ are (perfectly) Byzantine, the absorption probabilities satisfy:
    \be
    {\rr_b \ov \bb_{n-b}} > \prod_{j = n-b}^b {n-j \ov c-j} = {b! \ov (n-b-1)!} \times {(c-b-1)! \ov (b-f)!}~,
    \ee
    for $\ceil*{{n\ov 2}} \leq b \leq n-\alpha$, where $c = n-f$.
\end{lemma}

\begin{proof}  Using the above identities, the numerators of $\rr_b$ and $\bb_{n-b}$ can be simplified to
\bea
    \rr_{b}\, & \propto \sum_{l = \alpha - f}^{c-b}  \prod_{i=1}^{c-l} \mu_i \prod_{j=c-l+1}^{c-1}\lambda_j~, \\
    \bb_{n-b} & \propto \,\, \sum_{l = \alpha}^{n-b} \,\,\, \prod_{i=1}^{l-1} \, \mu_i \, \, \prod_{j=l}^{c-1} \, \lambda_j~, 
\eea
as $\lambda_{l < \alpha} = 0$ and $\mu_{l>n-\alpha} = 0$. Due to these identities, and since $n-\alpha < c$, we can rewrite the numerator of $\rr_b$ as
\bea
    \rr_{b}\,  \propto  \sum_{l = \alpha}^{n-b} \, \prod_{i=1}^{n-l} \mu_i \prod_{j=n-l+1}^{c-1}\lambda_j~.
\eea
Further algebraic manipulations of this expression lead to
\bea
    \rr_{b}\,  \propto  \sum_{l = \alpha}^{n-b} \, \prod_{i=1}^{l-1} \mu_i \prod_{j=l}^{c-1}\lambda_j \times \left( \prod_{t=l}^{n-l} {n-t \ov c-t}\right)~,
\eea
as long as $\alpha > f$. The fraction inside the last product ${n-t \ov c-t}$ is always greater than 1 for $f > 0$ and reaches its minimum for $l = n-b$. Using this value for all terms in the sum, one obtains the required identity.

\end{proof}

\section{Proofs of main results} \label{app: proofs}

\subsection{Proof of Theorem~\ref{thm: Slush vs delta-majority}} \label{app: proof of Slush vs delta-majority}

\begin{LateProof}\textit{of Theorem~\ref{thm: Slush vs delta-majority}.}\, To prove the theorem, we proceed as before by computing $\Delta \mathbb{P} = \mathbb{P}_{\rm Maj}^{(\delta)} - \mathbb{P}_S$, where $\mathbb{P}_{\rm Maj}^{(\delta)}$ is the majority rule success probability \eqref{hom Condorcet probability}, but with the starting point of the sum changed to $\ceil*{{n\ov 2}} + \delta$. We note that:
    \bea
    \Delta \mathbb{P} \leq & \sum_{b = \ceil*{{n\ov 2}}+\delta}^n \binom{n}{b}p^b(1-p)^{n-b}\rr_b  \,
    -  \sum_{b = \ceil*{{n\ov 2}}+\delta}^n \binom{n}{b-2\delta}p^{n+2\delta-b}(1-p)^{b-2\delta}\bb_{n+2\delta-b}~.
    \eea
    A term-by-term comparison reveals that a set of sufficient conditions for $\Delta\mathbb{P} < 0$ is
    \be
    {\rr_b \ov \bb_{n+2\delta - b}}\, \binom{n}{b}\binom{n}{b-2\delta}^{-1} \left( {p \ov 1-p} \right)^{2b-n-2\delta} < 1~,
    \ee
    for $\ceil*{{n\ov 2}} + \delta \leq b < n - \alpha$. Each such condition on its own implies that the participant accuracy $p$ must be below a threshold value  $p_{\rm th}(b; \delta) = 1 - \big(1+\tau(b;\delta)\big)^{-1}$~, with
    \be \label{tau(b)}
        \tau(b; \delta) = \left({\rr_{b-2\delta} \ov \rr_b}\times  \binom{n}{b-2\delta}\middle/ \binom{n}{b} \right)^{1\ov 2b-n-2\delta}~,
    \ee
   where we also use Lemma~\ref{lemma: Slush symmetry}.  When combining these constraints, we are looking for the value of $b$ that minimises $p_{\rm th}(b; \delta)$.\footnote{This value is only a lower bound on the threshold value $p_{\rm th}(\delta)$.}  For our purposes, however, this will not be required. Instead, note that $\tau(b; \delta) > 1$, and thus $p_{\rm th}(b; \delta) > {1\ov 2}$ for any $b$ and $\delta$. Thus, the true threshold $p_{\rm th}(\delta)$ will also be larger than ${1\ov 2}$, as claimed.
\end{LateProof}

\subsection{Proof of Theorem~\ref{thm: Slush convergence}}   \label{app: proof of Slush - Condorcet}

\begin{LateProof}\textit{of Theorem~\ref{thm: Slush convergence}.}\, This theorem consists of two separate statements. First, let us look at $\mathbb{P}_{\rm S}(n, k, \alpha, p) \geq p$, which is equivalent to showing
    \be
    \mathbb{P}_{\rm S}(n, k, \alpha, p) \geq {1\ov n} \sum_{b=0}^n b \,\binom{n}{b}p^b(1-p)^{n-b}~,
    \ee
    where the RHS is the first moment of the binomial distribution. Next, we split the sums from both sides in two, the first one containing the terms up to $\floor*{{n\ov 2}}$, and the second one having the remaining terms. After a simple change of variables, the above condition becomes:
    \bea
    & \sum_{b = \ceil*{{n\ov 2}}}^n \binom{n}{b} \bigg( \bb_b \, p^b (1-p)^{n-b} + \bb_{n-b} \, p^{n-b} (1-p)^{b}\bigg)  \\
    & \geq \sum_{b = \ceil*{{n\ov 2}}}^n \binom{n}{b} \bigg( {b\ov n} \, p^b (1-p)^{n-b} + {n-b \ov n}\, p^{n-b} (1-p)^{b}\bigg) ~.
    \eea
    Finally, comparing these expressions term-by-term, and using Lemma~\ref{lemma: Slush symmetry},  a sufficient condition for the result to be true is given as:
    \bea
    \left( \bb_b - {b\ov n}\right) \left( \Bigg( {p\ov 1-p}\Bigg)^{2b-n} - 1\right) \geq 0~,
    \eea
    for $b > {n\ov 2}$. This is true by Lemma~\ref{lemma: Slush symmetry}, as long as $p>{1\ov 2}$.

    The second statement of Theorem~\ref{thm: Slush convergence} involves the convergence of the Slush algorithm accuracy, $\mathbb{P}_S$, to unit for large enough $n$. This statement builds on Lemma~\ref{thm: Slush vs delta-majority}, leading to a much stronger result. First, according to the aforementioned lemma, large accuracies for the Slush algorithm can occur for the base learner accuracy $p$ in the interval $(q, p_{\rm th}(q))$. Nevertheless, using the monotonicity of the Slush algorithm proved in Lemma~\ref{lemma: Slush monotonicity}, we can eliminate the upper bound of this interval. 
    
    However, the only remaining issue is to show that $q < p_{\rm th}(q)$, for any $n$, such that Lemma~\ref{thm: Slush vs delta-majority} can apply. To enlarge the domain of base learner accuracies, we look at $\delta = 1$ or $q = {1\ov 2} + {1\ov n}$. Based on the proof of Theorem~\ref{thm: Slush vs delta-majority} -- and more precisely on \eqref{tau(b)} -- a sufficient condition for $p_{\rm th}(q) > q$ is
    \be \label{threshold sufficient condition}
         {\rr_{b-2} \ov \rr_b} \, {(b-1)(b-2) \ov (n-b+1)(n-b+2)}> \left({n+2 \ov n-2}\right)^{2b-n-2}~,
    \ee
    for all $\ceil*{{n\ov 2}} + 1 \leq b \leq n-\alpha$. Clearly, all fractions involved are greater or equal than 1. To prove this inequality, consider first the lowest value $b = \ceil*{{n\ov 2}} + 1$, when the exponent on the RHS is simply 1; meanwhile, for the LHS we extend Lemma~\ref{lemma: Slush symmetry}, such that:
    \bea
        \rr_{\floor*{{n\ov2}}} > 1 - {\floor*{{n\ov2}} \ov n}~, \qquad \quad \rr_{\ceil*{{n\ov2}} + 1} < 1 - {\ceil*{{n\ov2}} + 1 \ov n}~.
    \eea
    It follows that:
    \be \label{aux condition}
        {\rr_{\floor*{{n\ov2}}} \ov \rr_{\ceil*{{n\ov2}} + 1}} > {n+1 \ov n-3},
    \ee
    which is indeed greater than ${n+2 \ov n-2}$. We would like to present a proof of \eqref{threshold sufficient condition} that can easily generalise to other probabilistic consensus protocols. Thus, we want to avoid using too many details that are specific to the Slush protocol. For the Slush protocol, the ratio of the absorption probabilities on the LHS of \eqref{threshold sufficient condition} grows (very fast) with $b$. A detailed argument for this claim is presented below.

    Consider the limiting case $\alpha = 1 = k$, where from Lemma~\ref{lemma: Slush symmetry} we have $\rr_b = 1- {b\ov n}$ for any $0 \leq b \leq n$. As a result, the ratio ${\rr_{b-2} \ov \rr_b }$ reads ${n-b+2\ov n-b}$, which can be shown to be strictly increasing. Additionally, this ratio is strictly larger than ${n+2 \ov n-2}$  for $b > \ceil*{{n\ov 2}}$. When $\alpha$ and $k$ are varied away from this configuration, the rate of change\footnote{We define the rate of change as the difference in consecutive absorption probabilities, \ie $\rr_{b} - \rr_{b+1}$.} in the absorption probability becomes sharper in the region around $b \approx {n\ov 2}$ and milder otherwise, as depicted in Figure~\ref{fig: Slush Absorption}. Thus, as long as this is the case, we have $\rr_{b-2} - \rr_b > {2\ov n}$ and thus
    \be
        {\rr_{b-2} \ov \rr_{b}} > 1 + {2 \ov n \rr_b}~.
    \ee
    The RHS is then larger than ${n+2 \ov n-2}$ as long as $\rr_b < {1\ov 2} - {1\ov n}$, which is of course true for $b > \ceil*{n\ov 2}$. Importantly, the above argument will hold when the rate of change in the absorption probabilities is larger than ${1\ov n}$. However, we are concerned with large values of $n$, and thus ${1\ov n}$ can be made arbitrarily small. The net effect of this is to extend the region where the above argument holds. Finally, for the tail values, \ie when $b$ is large and the rate of change is smaller than ${1\ov n}$, the absorption probability $\rr_b$ converges to 0, and thus the ratio $\rr_{b-2} / \rr_{b}$ diverges.\footnote{We are only interested in $b < n-\alpha$, where the strict inequality $\rr_{b-2} > \rr_b$ holds. Otherwise, $\rr_b = 0$ for $b > n-\alpha$.}

    In light of the above reasoning, we can use the ratio $\rr_{\floor*{{n\ov2}}}/ \rr_{\ceil*{{n\ov2}} + 1}$ as a placeholder for any value of $b$ in the interval. Then, \eqref{threshold sufficient condition} simplifies upon using \eqref{aux condition}  to:
    \be
         {(b-1)(b-2) \ov (n-b+1)(n-b+2)} > \left({n+2 \ov n-2}\right)^{2b-n-3}~,
    \ee
    for $b \geq \ceil*{{n\ov 2}} + 2$. An equivalent way of writing this is
    \be
        {(n+2j+1)(n+2j+3) \ov (n-2j-1)(n-2j-3)} > \left({n+2 \ov n-2}\right)^{2j+2}~,
    \ee
    for $0 \leq j \leq {1\ov 2}(n-2\alpha-5)$. For $j = 0$, the identity can be proved by expanding the terms and comparing the resulting quartic polynomials. More generally, the proof can be done by induction, using
    \bea
        {(n+2j+1)(n+2j+3) \ov (n-2j-1)(n-2j-3)} & = {(n+2j-1)(n+2j+1) \ov (n-2j+1)(n-2j-1)}\times {(n+2j+3)(n-2j+1) \ov (n-2j-3)(n+2j-1)} \\
        & > \left({n+2 \ov n-2}\right)^{2j} \times {(n+2j+3)(n-2j+1) \ov (n-2j-3)(n+2j-1)} \\
        & > \left({n+2 \ov n-2}\right)^{2j} \times \left({n+2 \ov n-2}\right)^{2}~,
    \eea
    where in the last line one proceeds as for the $j = 0$ case by expanding the brackets and computing the quartic polynomials. 
    
\end{LateProof}

\subsection{Proof of Theorem~\ref{thm: performance groups}} \label{app: proof of performance groups}

\begin{LateProof}\textit{of Theorem~\ref{thm: performance groups}.}\, The distribution for the initial number of blue nodes follows a \emph{Poisson binomial distribution}, given by:
\bea \label{2 PBN}
    \mathbb{P}_{\rm PBN}(S_n=b) & =  \sum_{j = \max(0, b-n_2)}^{\min(b, n_1)} \binom{n_1}{j} \,p_1^{j}(1-p_1)^{n_1-j}  \times  \binom{n_2}{b-j}\, p_2^{b-j}  (1-p_2)^{n_2 - b + j}~.
\eea
When $p_2 = 0$, the second part of the expression should be neglected, and the sum reduces to a single term $j = b$. For now, consider generic values of $p_1, p_2$. As in the proof of Theorem~\ref{thm: Slush vs Majority hom}, we do a term-by-term analysis, leading to the control ratio:
\be \label{PBN control ratio}
     \kappa_b = {\mathbb{P}_{\rm PBN}(S_n=b) \ov \mathbb{P}_{\rm PBN}(S_n=n-b)}~,  \quad \text{for } m+1 \leq b \leq n-\alpha~,
\ee
where $n = 2m+1$. We are interested in finding the values of $p_1$ for which $\kappa_b = 1$ and below which $\kappa_b < 1$, for all $b$. The largest such value occurs in the limiting case $p_2 = 0$ when we have:
\bea
    p_{1, \rm max}(b) = \left( 1 + \tau_b^{1\ov 2b-n} \right)^{-1}~.
\eea
Here we introduced:
\be
    \tau_b = \binom{n_1}{b}\binom{n_1}{b-n_2}^{-1}~,
\ee
with $m+1 \leq b \leq n_1$. Note also that $0 \leq \tau_b \leq 1$ for any $b$ in the given range, and thus $p_{1, \max} > {1\ov 2}$. The result of the theorem is obtained using $b = \ceil*{{n\ov 2}}$ in the previous expressions, which is the value minimising $\kappa_b$. To show that this is indeed the case, let $\lambda(b) = \tau_b^{1\ov 2b-n}$, and note that:
\bea
     \left({\lambda(b) \ov \lambda(b+1)}\right)^{(2b-n)(2b-n+2)}  = \left( {b+1 \ov b+1-n_2}\,\,{n - b \ov n_1 - b}\right)^{2b-n} \binom{n_1}{b}^2 \binom{n_1}{b-n_2}^{-2}~,
\eea
which is clearly greater than 1 for any $b \geq \ceil*{{n\ov 2}}$ as long as $n_2>0$. Thus, $\lambda(b)$ is strictly decreasing, as claimed.

Finally, for $p_2 = {1\ov 2}$, one can check that $p_1 = {1 \ov 2}$ ensures that $\kappa_b = 1$. This follows from the identity \eqref{hypergeometric distr normalisation}. Thus, for more general $p_2 \in \left( 0, {1\ov 2}\right)$, the sought-after value of $p_1$ will lie between the above two limiting cases. 
    
\end{LateProof}

\section{Simulation details} \label{app: simulation details}

In this appendix, we discuss certain aspects of the simulations not covered in Section~\ref{sec: simulation}.

\medskip \paragraph{Training sets distributions.} The extended FEMNIST dataset consists of handwritten digits, lower case letters and upper case letters, partitioned by the writer of the digit or character. To formulate a binary classification problem, we discard the digits and only consider lower and upper case letters. This, of course, reduces the size of the datasets of individual writers. The distribution of dataset sizes used for training the simple ensemble methods (RF, XGBoost and LGBM models) can be seen on the left diagram of Figure~\ref{fig: RF training size}.

\begin{figure}[!ht] %
    \centering
    \includegraphics[width=0.49\textwidth]{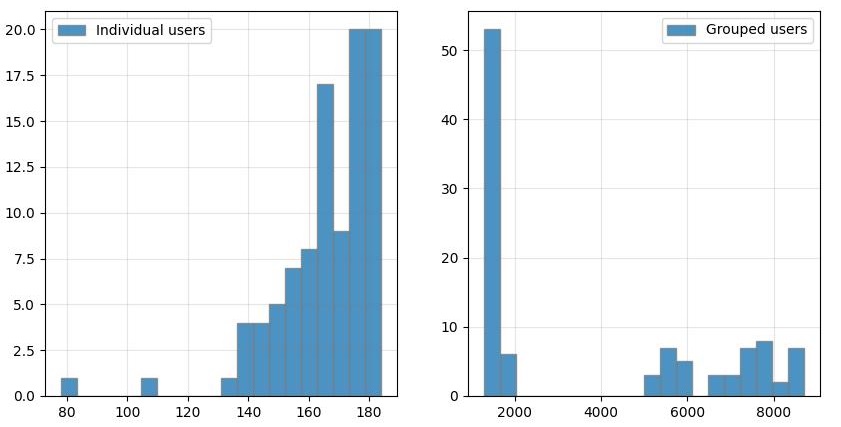}
    \caption{Distribution of dataset size used for training the 101 models. \emph{Left:} datasets of 101 individual users. \emph{Right:} datasets of 101 groups of users deployed for building CNNs.}\label{fig: RF training size}
\end{figure}%

For training CNNs, we group the data of the individual users to create 101 larger datasets. In the LEAF benchmark~\cite{caldas2019leaf}, the 3550 non-iid sets are split into 36 files of 100 users each, with the 50 users in the last file. For simplicity, we split each such file into three groups of equal numbers of users; more precisely, we have groups of 33, 33 and 34 users. The corresponding datasets are then combined in larger sets. The sizes of these training sets are also shown in Figure~\ref{fig: RF training size}. Note that the test datasets are separate from these training datasets.

\medskip \paragraph{Selecting Byzantine participants.} The simulations on the FEMNIST dataset described in Section~\ref{sec: simulation} do also include Byzantine participants. Consider first the 101 models trained on data from 101 distinct users. For this first scenario, we choose models at random, which are then turned into perfect Byzantine models, \ie their outputs are always incorrect. We then evaluate the accuracy of the majority ensemble and of the Slush algorithms. 

The results shown in Figure~\ref{fig: feMNIST letters niid Byzantine} consider two particular samples of base learners turned into perfect Byzantine models: strong classifiers and weak classifiers. By this we mean that the performances of the classifiers on the test sets are above (below, respectively) the average. %
\begin{table}[!ht]
    \centering
    \begin{tabular}{llll@{\hskip 0.3in}lll}
        \toprule 
        & \multicolumn{3}{c}{\textbf{Strong classifiers}} & \multicolumn{3}{c}{\textbf{Weak classifiers}} \\ 
       \textbf{Statistics}  & Test (1) & Test (10) & Test (100) & Test (1) & Test (10) & Test (100)  \\ \midrule
       Mean  & 0.699 & 0.610 & 0.611 & 0.601 & 0.552 & 0.554\\
       Median & 0.714 & 0.620 & 0.619 & 0.531 & 0.521 & 0.523\\
       Std & 0.091 & 0.037 & 0.037 & 0.114 & 0.057 & 0.053\\
       \bottomrule
    \end{tabular}
    \caption{Statistics for the two (random) samples of 10 base learners (Random Forests) turned into Byzantine models. Here, each model is trained on data coming from a single user.}
    \label{tab: stats RF Byz users}
\end{table}%
These statistics are described in Table~\ref{tab: stats RF Byz users} for samples of 10 users, and should be compared with these indicated on the top row of Figure~\ref{fig: feMNIST letters niid} for the whole ensemble of 101 classifiers.

A similar procedure is applied to the second simulation, where base learners are trained on data obtained from multiple users. There, we only consider a sample of 10 base learners that is representative of the whole ensemble of 101 models. Their statistics are shown in Table~\ref{tab: stats CNN Byz users}, %
\begin{table}[!ht]
    \centering
    \begin{tabular}{llll}
    \toprule
      \textbf{Statistics}   & Test (1) & Test (10) & Test (100)  \\ \midrule
       Mean accuracy  & 0.800 & 0.711 & 0.715 \\
       Median accuracy & 0.816 & 0.726 & 0.727 \\
       Standard deviation & 0.090 & 0.044 & 0.044\\
       \bottomrule
    \end{tabular}
    \caption{Statistics for the sample of 10 base learners (CNNs) turned into Byzantine models.}
    \label{tab: stats CNN Byz users}
\end{table} %
while the ensemble specifications are shown on the top row of Figure~\ref{fig: feMNIST letters niid CNN}.

\medskip \paragraph{Error measurement.} The plots shown in Section~\ref{sec: simulation} typically include the estimates for the error in ensemble accuracy. The first error to consider is due to the limited number of samples $n_{\rm samples}$ of the test set, being explicitly given by $\varepsilon = \left(n_{\rm samples}\right)^{-1}$. This error manifests for any single accuracy estimate of the Slush algorithm. To reduce this error, the communication phase is repeated $\cN = 10$ times for the testing on the datasets coming from 1 and 10 users. Meanwhile, for the test set of 100 grouped users, $n_{\rm samples}$ is large enough for this error to be negligible. 

To see how this repetition affects the overall error in the estimate for accuracy of the Slush algorithm we will proceed as follows. Let us model the ensemble accuracy as a random variable $Y$, whose variance $\sigma$ is closely related to the previously mentioned error $\varepsilon$ through $\sigma \propto \varepsilon^2$. The exact relation is not relevant for our purposes. The next step is to look at the mean of $\cN$ such random variables, whose variance becomes:
\be
    \sigma_{\bar Y} = {1\ov \cN^2}\,{\rm Var}\left(\sum_{i=1}^N Y_i\right)~.
\ee
If the $Y_i$ variables are uncorrelated, then we have $\sigma_{\bar Y} = \sigma / \cN$ and thus the error would reduce by a factor of $\sqrt{\cN}$. However, these variables are typically not uncorrelated. In the opposite case where the variables are perfectly correlated, we have instead $\sigma_{\bar Y} = \sigma$. This latter case applies to the majority rule, which is instead deterministic.

More generally, the error in Slush ensemble accuracy will be smaller than $\varepsilon$, but larger than $\varepsilon/\sqrt{\cN}$, and can be found by explicitly computing the covariance matrix of the $Y_i$ variables.

\medskip \paragraph{Beta distribution generalities.} The mean and variance of a random variable $Z \sim {\rm B}(a,b)$ can be expressed as:
\be
    \mathbb{E}[Z] = {a \ov a+b}~, \qquad {\rm var}(Z) = {ab \ov(a+b)^2(1+b+1)}~,
\ee
which can be easily inverted. As a result, an alternative way of fully specifying the beta distribution is through its mean and variance. Note that $\sigma^2 \equiv {\rm var}(Z)$ is limited to the interval $(0, 0.25)$. 
\printbibliography

\end{document}